\documentclass{article}

\usepackage[final]{Info_Learning}

\usepackage[utf8]{inputenc} % allow utf-8 input
\usepackage[T1]{fontenc}    % use 8-bit T1 fonts
\usepackage{hyperref}       % hyperlinks
\usepackage{url}            % simple URL typesetting
\usepackage{booktabs}       % professional-quality tables
\usepackage{amsfonts}       % blackboard math symbols
\usepackage{nicefrac}       % compact symbols for 1/2, etc.
\usepackage{microtype}      % microtypography
\usepackage{graphicx}

\newcommand{\I}{\mathrm{I}}
\newcommand{\E}{\mathrm{E}}
\newcommand{\Var}{\mathrm{Var}}
\newcommand{\argmax}{arg\,max}
\newcommand{\argmin}{arg\,min}
\newcommand{\arginf}{arg\,inf}
\usepackage{amsthm}
\newtheorem{claim}{Claim}[section]

\usepackage{xcolor}

\title{Information Theoretic Interpretation of Deep learning}

\author{
  Tianchen Zhao \\
  Department of Mathematics\\
  University of Michigan\\
  Ann Arbor, MI 48104 \\
  \texttt{ericolon@umich.edu} 
}

\begin{document}

\maketitle

\begin{abstract}
  We interpret part of the experimental results of \citet{shwartz2017opening}. Inspired by these results, we established a conjecture of the dynamics of the machinary of deep neural network. This conjecture can be used to explain the counterpart result by \citet{michael2018on}.
\end{abstract}

\section{Introduction}

\subsection{Set up}
Consider a deep learning problem, the training data set $\mathcal{D}=\{(X_i, Y_i)\}_{i=1:n}$ is known, where $\{X_i\}_{i=1:n}$ are the objects of interest, and $\{Y_i\}_{i=1:n}$ are the corresponding labels, sampled from random variables $(X,Y)$ with unknown joint distribution. For example, $X$ could be the continuous random variable over $[0,1]^n$ for some large integer $n$, representing the image of a single digit, and $Y$ is a discrete random variable with integer ranging from 0 to 9. 

In practice, the data $\mathcal{D}$ is fed into a deep neural network $\mathcal{F}$ parametrized by $w$ which is of very high dimension. For each object $x$, $\mathcal{F}(\tilde{y}|x;w)$ assigns $\{\tilde{Y}=\tilde{y}\}$ a probability, and the prediction is taken from the value with highest probability. In practice this is usually achieved by using softmax function. The goal is to use an optimization algorithm to train the parameter $w$ of $\mathcal{F}$ such that the joint probability of$(X,\tilde{Y})$ matches the joint probability of $(X,Y)$.

\subsection{Motivation}
Our work is primarily motivated by \citet{45820} and \citet{3266}. In an underdeterministic problem where the number of parameter $w$ is way larger than the amount of samples $\mathcal{D}$, there are typically infinite many solutions for $w$. According to {\it Occam's razor}, "simple" solutions are usually desired. In practice, the training of a deep neural network is not explicitly regularized, so there's no obvious guarantee that the solution $w$ we trained is "simple". However, the experimental results report that the model is still steadly improving after the set of solutions have already been reached. It's commonly believed that the optimization algorithm used, named stochastic gradient descent, is improving $w$ among the solutions of the underdetermined system.

The experiments performed by \citet{shwartz2017opening} give an answer to this phenomina from an information theoretic perspective. In their experiment, they designed a binary classification problem and trained a fully connected feed-forward neural network using SGD and cross-entropy loss function.

For each epoch during the training, they discretized the values for the last feature layer $\tilde{Y}$ into 30 bins over -1 and 1(which follows from the sigmoid activation function). Then $P(\tilde{Y},X)$ can be directly approximated by running the neural nets. Then they compute
$$P(\tilde{Y}|x) = \frac{P(\tilde{Y},x)}{P(x)},$$
which is used to compute
$$P(\tilde{Y}, Y)=\sum_x P(x,Y)P(\tilde{Y}|x,Y)=\sum_x P(x,Y)P(\tilde{Y}|x).$$

$P(\tilde{Y},X)$ and $P(\tilde{Y}, Y)$ are all we need to estimate $\I(\tilde{Y};X)$ and $\I(Y;\tilde{Y})$.

Their goal is to estimate the dynamics of the mutual information $\I(\tilde{Y};X)$ and $\I(Y;\tilde{Y})$. Here $\I(\tilde{Y};X)$ can be interpreted as a measure of how much information $\tilde{Y}$ retains("encodes") from $X$, and similarly $\I(Y;\tilde{Y})$ is how much information $\tilde{Y}$ preserves("decodes") from $Y$. We encourage the reader to watch the fantastic video of the optimization process in the \textit{information plane} at \textit{https://goo.gl/rygyIT}.

In this paper we are interested in the behaviour of their last layer(in orange), which is essentially the behaviour of $\I(X;\tilde{Y})$ and $\I(Y; \tilde{Y})$. At the beginning of the training, both $\I(X;\tilde{Y})$ and $\I(Y;\tilde{Y})$ increases, meaning the network is "memorizing" the data from $X$ and outputing more meaningful information to $Y$, note that this process is fast. Then $\I(X;\tilde{Y})$ starts to decrease while $\I(Y;\tilde{Y})$ keeps increasing, meaning the network is "forgetting" the data it just memorized but is still improving to give more information about $Y$, note that this process is slow. The first part is named the fitting phase and the second part is named the compression phase.

\citet{michael2018on} started a similar line of work measuring the dynamics of the mutual information. They proposed the following:

1. The compression phase is not explicit if the network uses ReLU activation function instead throughout the network.

2. The compression phase is happening alongside with the fitting phase if the input data is Gaussian-like and the labels are assigned randomly.

3. It is theoretically implausible to measure the information in the intermediate layer, given that mutual information between continuous random variables is in general ill-defined.

\subsection{Our contribution}

\subsubsection{Decomposition of Deep Network}

We propose the following conjecture of the machinary of deep neural network(see Figure~\ref{fig:network}): the data is transformed into linearly seperable feature by the nonlinear ReLU network, which is "almost invertible" as to be explained in detail below. The following full connected layer and sigmoid/softmax function together behave like a multiclass SVM, giving the best prediction among all available classes.

\begin{figure}[ht]
%\vskip 0.05in
\centerline{\includegraphics[width=4in]{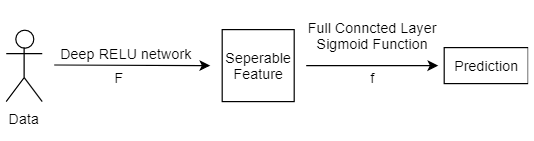}}
\caption{Our proposed decomposition of Deep Neural Network}
\label{fig:network}
\vskip 0.1in
\end{figure}

\subsubsection{Interpretation of Information Theory}

Given this structure, we can give interpretation to the experimental results by \citet{shwartz2017opening}. The fitting phase corresponds to finding parameters from ReLU network to organize the features in last layer into a linearly seperable fashion, and find full connected parameters represent hyperplanes to seperate the features. The compression phase corresponds to finding parameters maximizing the margin of the linear seperation, which is essentially driven by the randomness of the SGD algorithm. This progresses slowly as we will prove in Section~\ref{proof:SGD}. Note that the procedure of the margin maximization is to find the "support feature" and "forgetting" other features, enabled by the sigmoid/softmax function. It follows that mutual information between input $X$ and prediction $\tilde{Y}$ is diminising during this phase.

The phenomena propose by \citet{michael2018on} can therefore be explained:

1. If the activation function used in last layer is replaced by ReLU, then the SVM structure will be destroyed. The network is not "forgetting" approximately half of the features, so the compression phase will not be explicit.

2. The Gaussian input is isotropic so no network can transform this data into a linearly seperable one. We argue that compression phase only exists if there's a margin between data from different classes.

3. We agree that it's often dangerous to define mutual information between continuous valued random variables. In general, if the joint probability of two continuous random variable is not degenerate over an open neighborhood, then there exists an invertible mapping, which needs infinite amount of information to describe the mutual relationship. In fact, we do believe that the neural network will have better performance if the intermediate layers can fully preserve the information of $X$. In this case ReLU does a better job than sigmoid or tanh functions do, which matches the experimental reports by \citet{Nair:2010:RLU:3104322.3104425}. But the last layer should be fully connected and activated by sigmoid function, as indicated by \citet{NIPS2012_4824}.

However, the mutual information between continuous valued random variable and discrete valued random variable is well defined in a sense that the discretized measurement converges(see Appendix~\ref{ctsinfo}). So we argue that the expirical measurement by \citet{shwartz2017opening} on the last layer remains theoretical valid, which supports our hypothesis. Throughout our theoretical analysis below, both $Y$, $\tilde{Y}$ represent a discretized prediction random variable when comparing with $X$.

\subsubsection{Interpretation of Res-Net}

Residual Network by \citet{DBLP:journals/corr/HeZRS15}, the winner of ILSVRC2015, is one of the best existing deep network structures. Here we explain why it is so successful from our theory.

It's well known that a network that is too deep is not working very well. One of the reason could be, from our perspective, there's too much information loss from the first part, $F$, of our proposed structure (Figure~\ref{fig:network}). Res-Net is designed to allow the model to "learn" identity map easily. Specifically, in \citet{PhysRevE.69.066138}, they mentioned in appendix that mutual information is fully preserved under homeomorphisms (smooth and uniquely invertible maps). In Res-Net, the building block with input vector $x$ and output vector $y$ is related as:
\begin{equation} \label{eq:buildblock}
y = \mathcal{L}(x)+x=(\mathcal{L}+I)(x),
\end{equation}
where the operator $\mathcal{L}$ could be a composition of activation functions, convolution, drop-out(\citet{JMLR:v15:srivastava14a}) and batch normalization(\citet{DBLP:journals/corr/IoffeS15}). See Section \ref{Kaiming} for more details.

It can be shown that (see Appendix~\ref{operator} for a proof) if the operator norm $|\mathcal{L}|<1$, then $\mathcal{L}+I$ is theoretically guaranteed to have an inverse, which enables information preservation between intermediate layers. In Section~\ref{blocknorm} we experimentally verified that $|\mathcal{L}|<1$ for all intermediate layers.

The main strength of Res-Net, in our understanding, is it allows the deep mapping $F$ in Figure~\ref{fig:network} to be invertible, independent of $\mathcal{L}$, which is almost never invertible due to the use of ReLU, convolution, drop-out and other singular operations.

In the work by \citet{jacobsen:hal-01712808}, they built an deep invertible networks and showed that a network could be successful without losing any information in the intermediate layers.

\subsection{Related work}
\citet{zheng2018understanding} understood DNN from a Maximum entropy perspective. \citet{DBLP:journals/corr/ChaudhariCSL16} used entropy to detect wide valleys for SGD algorithm. \citet{SHAMIR20102696} proved generalization property of IB framework. \citet{journals/corr/AchilleS16} investigated the amount of information loss through various of operations in deep network. \citet{article} proved an upper bound on the number of nodes needed for a two-layer network to seperate the data.  \citet{DBLP:journals/corr/AlemiFD016} designed a variational approximation to the IB framework.

\section{Proofs}

In this section we prove our conjecture in Figure~\ref{fig:network} as follows:

\begin{figure}[ht]
%\vskip 0.05in
\centerline{\includegraphics[width=4in]{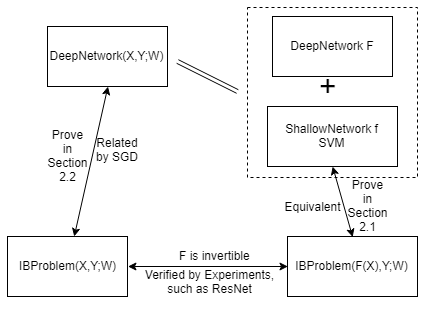}}
\caption{Chart of logic for our proof.}
\label{fig:logic}
\vskip 0.1in
\end{figure}

1. We prove that it takes $\I(Y; \tilde{Y})$ a time with polynomial order to grow during the fitting phase, and a time with exponential order to grow during the compression phase, which matches the video of the information plane dynamics by \citet{shwartz2017opening}.

2. We assume that $F$ in Figure~\ref{fig:network} is almost invertible in the sense that $\I(F(X);Y) \approx \I(X;Y)$.

We argue that the Deep Network $\mathcal{F}$, treated as an abstract function describing relationship between $\tilde{Y}$ and $(X,Y)$, is governed, if the word "parametrized" is not proper, by the quantities $\I(X; \tilde{Y})$ and $\I(Y; \tilde{Y})$.

From an information theoretic set up, the IB problem can be considered as:
\begin{equation}
\I(X;\tilde{Y})-\alpha \I(Y; \tilde{Y})  =  \I(F(X);\tilde{Y})-\alpha \I(Y; \tilde{Y}),
\end{equation}
where $\alpha$ is some positive constant.
Therefore we can abuse the notation between $X$ and $F(X)$ as they are equivalent in the information theoretic setting. In particular, we have $\tilde{Y}=f(X)$ where $f$ is defined in Figure~\ref{fig:network}.

Then we prove there's a direct relationship between IB problem and SVM problem in linear case.

We conclude that deep neural network, from a information theoretic point of view, can be reduced to a hard-margin SVM problem.

3. We discuss the notion of generalization in our context.

\subsection{Machinary of SGD}
\label{proof:SGD}

There is a history of work towarding the behaviour of discrete SGD. In particular, \citet{roberts1996} shows that the discrete Langevin difussion converges to the target distribution exponentially fast in time. This is a general result applying to Markov Chain Monte Carlo(MCMC) and is not practical because SGD is a "path-finder" with a diminishing step length. \citet{pmlr-v65-raginsky17a} and \citet{DBLP:journals/corr/abs-1802-06439} showed that the searching path of SGD needs time of exponential order to jump from one local min to another local min, which is intuitively why the compression phase is slow. \citet{Zhang2017AHT} showed that the SGD path needs time of polynomial order to enter the first "good" local min, which is intuitively why the fitting phase is fast.

In our work, we present a similar result from an information theoretic perspective, by using an analogy to standard stochastic convex analysis in \citet{Bottou:1999:OLS:304710.304720}.

Recall this video from \citet{shwartz2017opening} on \textit{information plane} at \textit{https://goo.gl/rygyIT}.

\begin{claim}
The general form of SGD is given by:
\begin{equation}
w_{t+1} = w_t -a_t\nabla F(w_t)-b_tB,
\end{equation}
where $a_t,b_t$ are positive constant varying with time $t$ and $B$ is the standard Gaussian $\mathcal{N}(0,\I)$.

In particular, consider a Langevin Monte Carlo(LMC) setting where $a_t=\frac{1}{t}$ and $b_t=\frac{1}{\sqrt{t}}$.

Then $\I(Y;\tilde{Y})$ reaches local maximum in time with polynomial order, and hits higher local maximum in time with exponential order.
\end{claim}

\begin{proof}

From Appendix~\ref{proof(1)}, we have:
\begin{equation}
\I(X; \tilde{Y})  \leq 2(\Var(p(\tilde{y})) + \Var(p(\tilde{y}|x))).
\end{equation}

From Appendix~\ref{proof(2)}, we have:
\begin{equation} \label{eq.bound(3)}
\I(Y;\tilde{Y}) \geq A + \log(\prod_{i=1}^n f(y_i|x_i;w_t)) + \sum_{i=1}^n \log(p(y_i|x_i)).
\end{equation}

Now I focus on showing this lower bound (\ref{eq.bound(3)}) is first increasing polynomially fast, then logarithmically slow, driven by the machinary of SGD.

Denote $F(w_t) = -\sum_{i=1}^n \log(f(y_i|x_i,w_t))$ and consider the SGD mechanics in general form as follows:
\begin{equation}
w_{t+1} = w_t -a_t\nabla F(w_t)-b_tB,
\end{equation}
where $a_t,b_t$ are positive constant varying with time $t$ and $B$ is standard Gaussian $\mathcal{N}(0,\I)$.

Denote a discrete set of local minimas and saddle points of $F$ as $\mathcal{C}$.

Define the metric $\rho$ as:
\begin{equation}
\rho(w, \mathcal{C}) = \inf\{||w-c||_2: c \in \mathcal{C}\}.
\end{equation}

Assumption:

(1) $F$ is $M$ smooth with respect to $\mathcal{C}$:

Let $c_x=\arginf_c\{||x-c||_2: c \in \mathcal{C}\}$, then $||\nabla F(x)-\nabla F(x)||_2 \leq M||x-y||_2$, for all $x,y \in W$ satisfying $c_x = c_y$.

(2) $F$ is $m$ strongly convex with respect to $\mathcal{C}$:

Let $c_x=\arginf_c\{||x-c||_2: c \in \mathcal{C}\}$, then $\langle \nabla F(x)-\nabla F(x), x-y \rangle \geq m||x-y||_2^2$, for all $x,y \in W$ satisfying $c_x = c_y$.

Now suppose $w_1$ is closed to $w^* \in \mathcal{C}$, where $w^*$ is more than O(t) away from other element in $\mathcal{C}$, then if we consider the training process restricted to a polynomial time regime, we have $\rho(w_t, \mathcal{C}) = ||w_t - w^*||$, and we have the following:
\begin{equation}
\begin{array} {lcl}
||w_{t+1}-w^*||^2 & = & ||w_t -a_t\nabla F(w_t)-b_tB - w^*||^2\\
& = & ||w_t-w^*||^2 -2\langle w_t-w^*, a_t\nabla F(w_t)+b_tB \rangle+||a_t\nabla F(w_t)+b_tB||^2
\end{array}
\end{equation}

Taking expectation with respect to filtration $\mathcal{F}_t$ at time $t$ gives:
\begin{equation}
\begin{array} {lcl}
\E(||w_{t+1}-w^*||^2|\mathcal{F}_t) & = & ||w_t-w^*||^2 -2a_t\langle w_t-w^*, \nabla F(w_t) \rangle+a_t^2||\nabla F(w_t)||^2+b_t^2\\
& \leq & ||w_t-w^*||^2 -2a_tm||w_t-w^*||^2+a_t^2M^2||w_t-w^*||^2+b_t^2\\
& = & (1-2a_tm+a_t^2M)||w_t-w^*||^2+b_t^2
\end{array}
\end{equation}

Now consider a Langevin Monte Carlo(LMC) setting where $a_t=\frac{1}{t}$ and $b_t=\frac{1}{\sqrt{t}}$. Then (7) becomes:
\begin{equation} \label{eq:scale}
\begin{array} {lcl}
\E(||w_{t+1}-w^*||^2|\mathcal{F}_t) & = & (1-2\frac{m}{t}+\frac{M^2}{t^2})||w_t-w^*||^2+\frac{1}{t}\\
\E(||w_{t+1}-w^*||^2)& = & (1-2\frac{m}{t}+\frac{M^2}{t^2})\E(||w_t-w^*||^2)+\frac{1}{t}\\
\E(||w_{t+1}-w^*||^2)& \leq & \prod_{s=1}^t(1-2\frac{m}{s}+\frac{M^2}{s^2})\E(||w_1-w^*||^2)+\sum_{s=1}^t \frac{1}{s}\\
& = & \prod_{s=1}^t(1-2\frac{m}{s}+\frac{M^2}{s^2})\E(||w_1-w^*||^2)+O(\log(t)+1)
\end{array}
\end{equation}

According to L'Hopital's rule:
\begin{equation}
\begin{array} {lcl}
\lim_{s \rightarrow \infty}\frac{\log(1-2\frac{m}{s}+\frac{M^2}{s^2})}{\frac{1}{s}} & = & \lim_{s \rightarrow \infty}\frac{\frac{2s^{-2}m-2M^2s^{-3}}{1-2\frac{m}{s}+\frac{M^2}{s^2}}}{-s^{-2}}\\
& = & -2m.
\end{array}
\end{equation}
So $\log(1-2\frac{m}{s}+\frac{M^2}{s^2})$ is comparable with $\frac{1}{s}$, and $\sum_{s=1}^t \log(1-2\frac{m}{s}+\frac{M^2}{s^2})$ is comparable with $-\log(t)$ and therefore $\prod_{s=1}^t(1-2\frac{m}{s}+\frac{M^2}{s^2})$ is of order $O(\frac{1}{t})$.

So our conclusion is:
\begin{equation}
\E(||w_{t+1}-w^*||^2) = O(\frac{1}{t})\E(||w_1-w^*||^2)+O(\log(t))
\end{equation}

This shows under a polynomial time regime, $w_t$ is converging to a local min of $F$ linearly fast.

But if we consider the exponential time regime, the above analysis is invalid because the last term of (\ref{eq:scale}) is innegligible. There's a possibility that $w_t$ may jump to another $w^{**} \in \mathcal{C}$.

Here we conclude that the lower bound (\ref{eq.bound(3)}) of the quantity of interest converges to a local min in polynomial time, then switch higher to "better" local max in exponential time, which explains the video of Tishby's.
\end{proof}

\subsection{Relationship between IB and SVM}

For simplicity we prove the result for binary classification.

\begin{claim}
 Consider a binary classification problem where the data set $\mathcal{D}=(x_i,y_i)$ are linearly seperable. $p(\tilde{y}=1|x,w)$ is modeled by $f(x,w)=\sigma(w^tx)$ and $p(\tilde{y}=-1|x,w)=1-f(x,w)$. Recall that $\sigma(t)=\frac{1}{1+e^{-t}}$.

 Given a IB problem:
\begin{equation}
Minimize \textrm{\space\space} \I(X;\tilde{Y})-\alpha \I(Y; \tilde{Y})
\end{equation}
it can be formulated as hard margin SVM problem.
\end{claim}
\begin{proof}

From Appendix~\ref{proof(1)} we have that:
\begin{equation} \label{control}
\I(X;\tilde{Y}) \leq  2(\Var(p(\tilde{y})) + \Var(p(\tilde{y}|x)))
\end{equation}
Here we make an assumption that for the models we trained over time, it's output $\tilde{Y}$ is approximately uniform distributed over the finite labels. So we treat $p(\tilde{y_i})$ as constant for all $i$. In particular, we assume: $p(\tilde{y}=1)\approx p(\tilde{y}=-1)\approx \frac{1}{2}$, then the first term of (\ref{control}) on RHS can be controlled:
\begin{equation} \label{uniform}
\Var(p(\tilde{y})) \leq D,
\end{equation}
for some constant $D$.

We can bound the second term of (\ref{control}) on RHS by:
\begin{equation} \label{eq:dev}
\begin{array} {lcl}
\Var(p(\tilde{y}|x)) & \leq & \int_{\tilde{Y}} \int_{X}p(x,\tilde{y}) p(\tilde{y}|x)^2dxd\tilde{y}\\
& = & \int_{X}p(\tilde{y}=1|x)^3p(x)dx + \int_{X}p(\tilde{y}=-1|x)^3p(x)dx \\
& = & \int_{X}(\sigma(w^tx))^3p(x)dx + \int_{X}(1-\sigma(w^tx))^3p(x)dx \\
& = & \int_{X}(1-3\sigma(w^tx)+3\sigma(w^tx)^2)p(x)dx \\
& = & \frac{1}{4} + 3\int_{X} (\sigma(w^tx)-\frac{1}{2})^2 p(x)dx\\
\end{array}
\end{equation}

Consider the 1st order Taylor expansion:
\begin{equation}
\sigma(w^tx) = \sigma(0) + \sigma'(c)(w^tx) \leq \frac{1}{2}+\frac{1}{4}w^tx.
\end{equation}

Substitute it into (\ref{eq:dev}) to get:
\begin{equation}
\begin{array} {lcl}
\Var(p(\tilde{y}|x)) & \leq & \frac{1}{4} + 3\int_{X} (\sigma(w^tx)-\frac{1}{2})^2 p(x)dx\\
& = & \frac{1}{4}+\frac{3}{16}\int_{X} (w^tx)^2 p(x)dx \\
& = & \frac{1}{4}+(\frac{3}{16}\int_{X} x^2 p(x)dx)||w||^2
\end{array}
\end{equation}

To conclude, we have a bound of (\ref{control}) the form:
\begin{equation} \label{eq:bound1}
\I(X;\tilde{Y}) \leq  A+B||w||^2.
\end{equation}

From Appendix~\ref{proof(2)}, we approximate the Mutual information $\I(Y;\tilde{Y})$ by:
\begin{equation}
\begin{array} {lcl}
\I(Y;\tilde{Y}) &=& \sum_i \log(\frac{p(y_i, \tilde{y_i})}{p(y_i)p(\tilde{y_i})})\\
&=& \sum_{i=1}^n \log(\frac{\sum_{j=1}^n p(y_i,\tilde{y_i}|x_j)}{p(y_i)p(\tilde{y_i})})\\
&=& \sum_{i=1}^n \log(\frac{\sum_{j=1}^n p(y_i|\tilde{y_i},x_j)p(\tilde{y_i}|x_j)}{p(y_i)p(\tilde{y_i})})\\
 &=& \sum_{i=1}^n \log(\frac{\sum_{j=1}^n p(y_i|x_j)p(\tilde{y_i}|x_j)}{p(y_i)p(\tilde{y_i})}).
\end{array}
\end{equation}
with high probability.

Also note that $p(\tilde{y_i}|x_j)$ is given by the model $f(\tilde{y_i}|x_j;w)$.

where the prediction $\tilde{y_i}$ satisfies:
\begin{equation}
p(\tilde{y_i}=1|x_i) = f(x_i,w) = \sigma(w^tx_i).
\end{equation}

So $\I(Y;\tilde{Y})$ is now of the form:
\begin{equation} \label{eq:bound2}
\begin{array} {lcl}
\I(Y;\tilde{Y}) & \geq & A + \sum_{i=1}^n \log(\sum_{j=1}^n p(y_i|x_j)f(\tilde{y_i}|x_j;w))\\
& \geq & A + \sum_{i=1}^n \log(p(y_i|x_i)f(\tilde{y_i}|x_i;w))\\
& \geq & A + \sum_{i=1}^n \log(f(\tilde{y_i}|x_i;w)) + \sum_{i=1}^n \log(p(y_i|x_i))\\
& \geq & A + \sum_{i=1}^n \log(|\sigma(w^tx_i)-\frac{1}{2}|+\frac{1}{2}) + \sum_{i=1}^n \log(p(y_i|x_i))
\end{array}
\end{equation}
with high probability for some constant $A$.

Finnally we put (\ref{eq:bound1})\&(\ref{eq:bound2}) together:
\begin{equation}
\I(X;\tilde{Y})-\alpha \I(Y; \tilde{Y})  \leq  A'+B||w||^2-\alpha \sum_i \log(|\sigma(w^tx_i)-\frac{1}{2}|+\frac{1}{2}).
\end{equation}
This is a Lagrangian form of an optimal margin classifier.
\end{proof}

\subsection{Generalization}
Consider the target random variable $X$ with unknown distribution. We have a collection of samples of it, denoted as $\mathcal{D} = \{x_i\}_{i=1}^n$. Also denote the hypothesis class as $\mathcal{H}$.

The loss function is defined as:
\begin{equation}
l: X \times \mathcal{H} \mapsto \mathbb{R_+}
\end{equation}

And the risk:
\begin{equation}
\mathcal{R} = \E_X [l(x,f)].
\end{equation}

The optimal function is defined as:
\begin{equation}
f^* \in \argmin_{f \in \mathcal{H}} \mathcal{R}(f).
\end{equation}

Empirical risk minimization(ERM) is given as:
\begin{equation}
\hat{f} \in \argmin_{f \in \mathcal{H}} \hat{\mathcal{R}}(f),
\end{equation}
where $\hat{\mathcal{R}}(f)$ is given as:
\begin{equation}
\hat{\mathcal{R}}(f) = \frac{1}{n} \sum_{i=1}^n l(x_i, f).
\end{equation}

Consider the following decomposition:
\begin{equation}\label{PAC}
\begin{array} {lcl}
\mathcal{R}(\hat{f}) & = & \mathcal{R}(\hat{f})-\hat{\mathcal{R}}(\hat{f})+\hat{\mathcal{R}}(\hat{f})-\hat{\mathcal{R}}(f^*)+\hat{\mathcal{R}}(f^*)-\mathcal{R}(f^*)+\mathcal{R}(f^*)\\
& \leq & \sup_{f \in \mathcal{H}} |\mathcal{R}(f)-\hat{\mathcal{R}}(f)| + 0 + |\hat{\mathcal{R}}(f^*)-\mathcal{R}(f^*)|+ \mathcal{R}(f^*).
\end{array}
\end{equation}
Here in (\ref{PAC}), the second $\hat{\mathcal{R}}(\hat{f})-\hat{\mathcal{R}}(f^*)$ is bounded by zero by definition; $\hat{\mathcal{R}}(f^*)-\mathcal{R}(f^*)$ is small, guaranteed by the law of large number under the assumption that the number of samples $n$ is sufficiently large. $\mathcal{R}(f^*)$ is a constant depending on the hypothesis class $\mathcal{H}$.

The bound $\sup_{f \in \mathcal{H}} |\mathcal{R}(f)-\hat{\mathcal{R}}(f)|$ for $\mathcal{R}(\hat{f})-\hat{\mathcal{R}}(\hat{f})$ is typically controlled by the VC theory in the literature, see \citet{Vapnik:1995:NSL:211359}. But as pointed out in \citet{45820}, if the number of parameters is much larger than the number of samples, some form of regularization is needed to ensure small generalization error. \citet{Neyshabur2017ExploringGI} mentioned that a sharper bound can be obtained by making it dependent on the choice of $\hat{f}$.

We would like to provide more insights to generalization, by formally establishing relationship between generalization in deep learning and margin. In our model, the SVM in the last layer generalizes better if the margin is larger. In \citet{DBLP:journals/corr/NeyshaburBMS17}, they carefully analysized the notion of normed based control on margin with relation to generalization. In \citet{DBLP:journals/corr/BartlettFT17}, they also proved a bound on generalization with margin by using Rademacher complexity.

\section{Experiments}

In this section we interpret existing experimental reults using our theory. We also did experiment verifying features of our own interest on Res-Net.

\subsection{Deep Residual Network}
\label{Kaiming}

Here we use full reference to experimental results from \citet{DBLP:journals/corr/HeZR016}.

\renewcommand\arraystretch{1.3}
\setlength{\tabcolsep}{6pt}
\begin{table}[!htbp]
\caption{Classification error (\%) on the CIFAR-10 test set using different activation functions.}\label{tab:activations}
\centering
\fontsize{8pt}{1em}\selectfont
\begin{tabular}{l|l|c|c}
\hline
case & Fig. & ResNet-110 & ResNet-164 \\
\hline
\hline
original Residual Unit & Fig.~\ref{fig:activations}(a) & 6.61 & 5.93 \\
\hline
BN after addition & Fig.~\ref{fig:activations}(b) & 8.17 & 6.50 \\
\hline
ReLU before addition & Fig.~\ref{fig:activations}(c) & 7.84 & 6.14 \\
ReLU-only pre-activation & Fig.~\ref{fig:activations}(d) & 6.71 & 5.91 \\
\textbf{full pre-activation} & Fig.~\ref{fig:activations}(e) & \textbf{6.37} & \textbf{5.46} \\
\hline
\end{tabular}
\end{table}

\begin{figure}[!htbp]
\centering
\includegraphics[width=.99\linewidth]{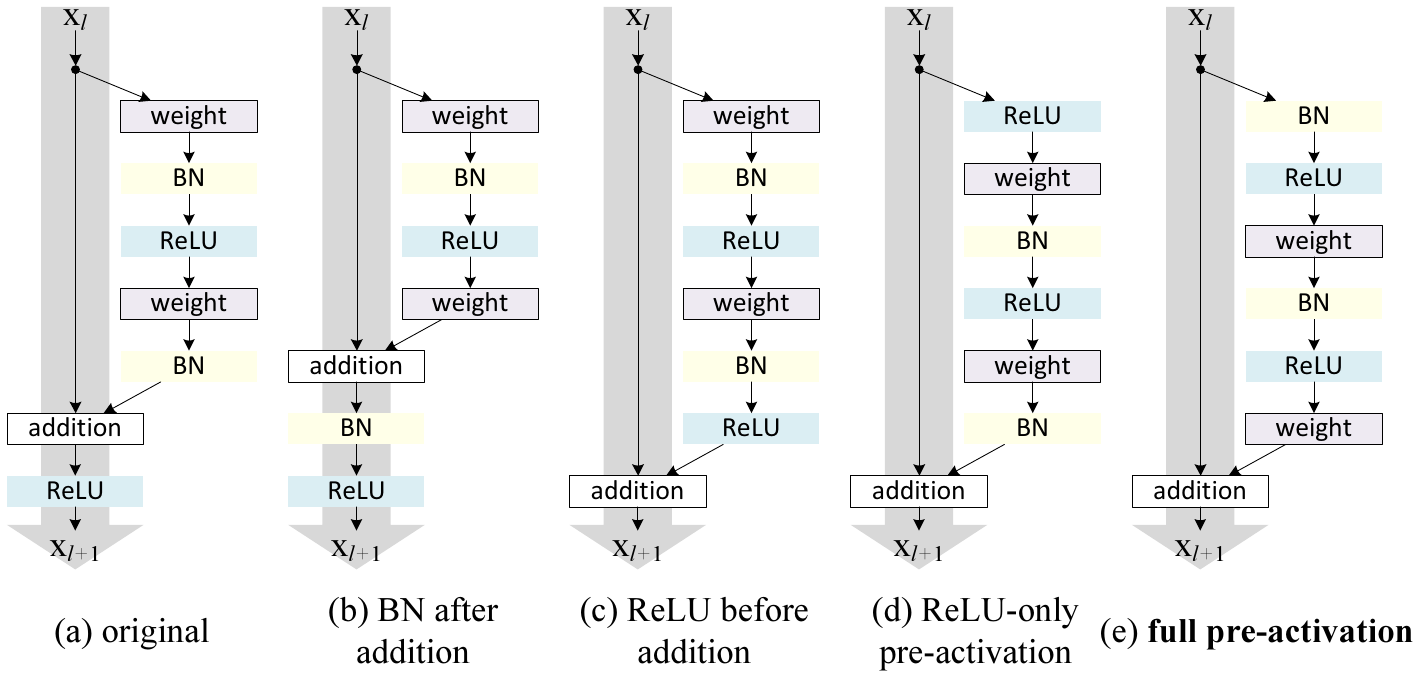}
\caption{Various usages of activation in Table~\ref{tab:activations}. All these units consist of the same components --- only the orders are different.}
\label{fig:activations}
\end{figure}
Here they compared different structures for building blocks and report the performance in Table~\ref{tab:activations}. We give an interpretation to their result according to our theory: (a)\&(b) both have ReLU after adding the identity mapping, which makes the whole mapping not invertible; (c) makes $\mathcal{L}$ defined in (\ref{eq:buildblock}) an nonnegative operator, which may potentially enlarge this operator norm out of our theoretical guarantee (Appendix~\ref{operator}); (d) performs ReLU directly to input, which loses information, applying Batch Normalization after convolution also is meaningless.

\subsection{Flat Basins}
\begin{figure}[ht]
  \centering
    \includegraphics[width=0.65\textwidth]{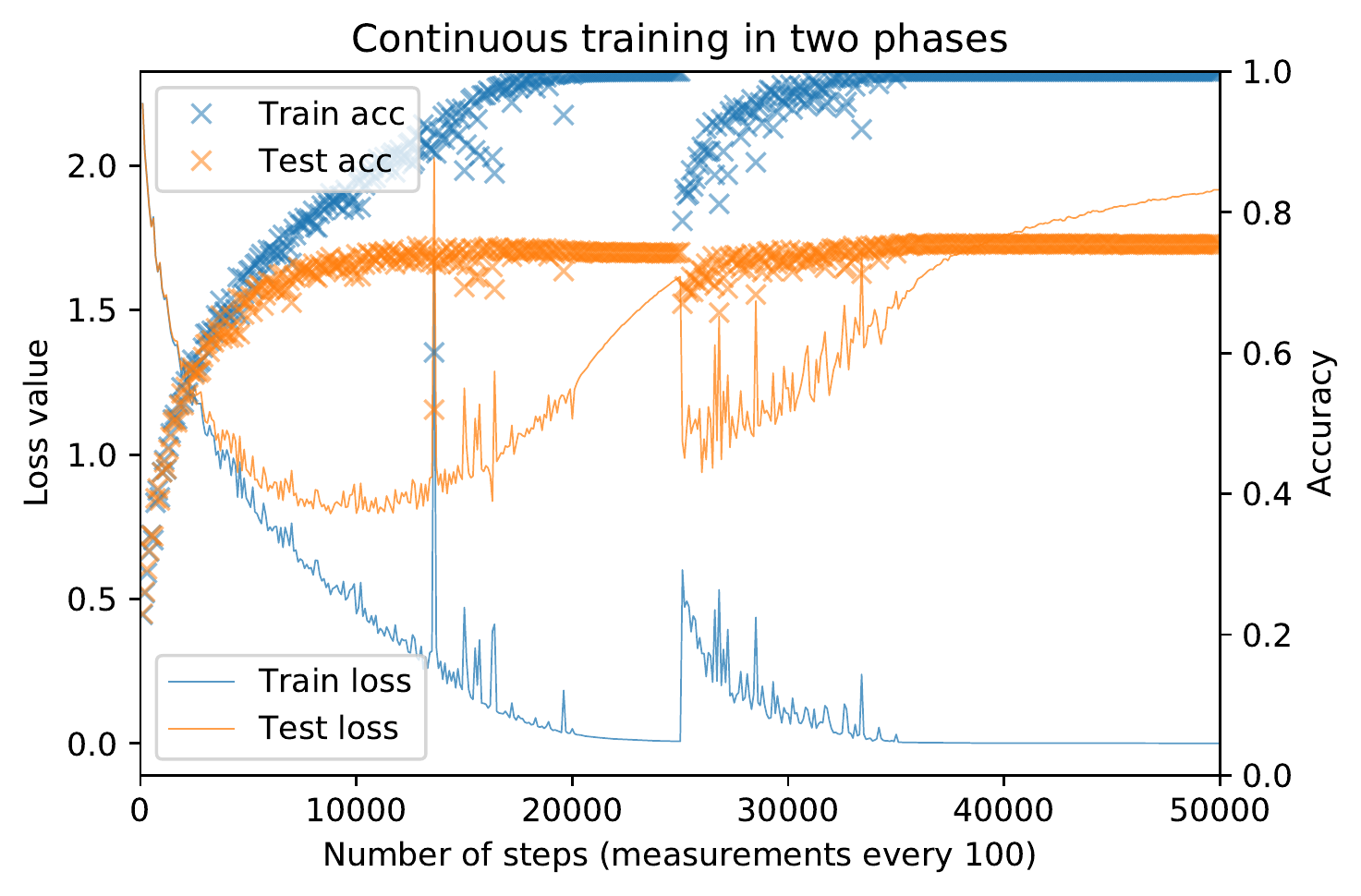}
    \caption{Large batch training immediately followed by small batch training on the full dataset of CIFAR10 with a \textit{raw} version of AlexNet.}
    \label{fig:LB_SB}
\end{figure}
Here we use full reference to experimental results from \citet{DBLP:journals/corr/SagunEGDB17}.

On generalization property of GD, \citet{3266} proved in appendix that GD converges to optimal solution to underdetermined least square problem, under proper initialization; \citet{soudry2018the} proved that GD converges to the optimal solution of a hard margin SVM slowly.

Their work motivated us to see the corresponding result in deep learning. The difference is, in our theory, the seperable data is not deterministic as it's controlled by the feature learning map $F$ in Figure~\ref{fig:network}. The network is looking for a $F$ to create a larger margin between data of different category and a "weighted" linear seperator $f$ to achieve that maximum margin.

Back to their experiment, they trained a neural network with batch gradient descent(GD) for the first 2.5k steps and switch to stochastic gradient descent(SGD) for the rest 2.5k steps. Notice there's a significant jump at step 2.5k when the optimization algorithm is changed.

They argue that despite the jump, there exists a linear interpolation between LHS and RHS so GD and SGD lead to essentially the same "basin". As pointed out by \citet{DBLP:journals/corr/DauphinPGCGB14}, in a very high dimensional problem, it's very hard to encounter a strict local minima: almost all critical points are saddle points as there could always exist some direction that is not "going up" in the landscape. So we do suspect that almost all "basins" are connected together in some sense. In the work by \citet{DBLP:journals/corr/DinhPBB17}, they argued that a universal definition of flatness of the error surface is still unclear. In particular, they proved that the geometry of a local minimum can be changed arbitrarily by reparametrization, without changing what function it represents.

Instead concerning about the notion of flatness, we can talk about margin. According to our theory, GD only seperates the feature data, but does not maximize the margin. At step 2.5k, there are some feature point sitting right at the boundary of the linear seperator and SGD will break this balance, which leads to that sudden jump in the training error. In the end SGD generalizes better because the linear seperator in the last layer seperates the feature data by a larger margin. The notion of margin needs to be defined carefully as it scales with the norm of the weights and data.

\subsection{Norm of $\mathcal{L}$}
\label{blocknorm}
\begin{figure}[!htbp]
\begin{center}
\includegraphics[width=.3\linewidth]{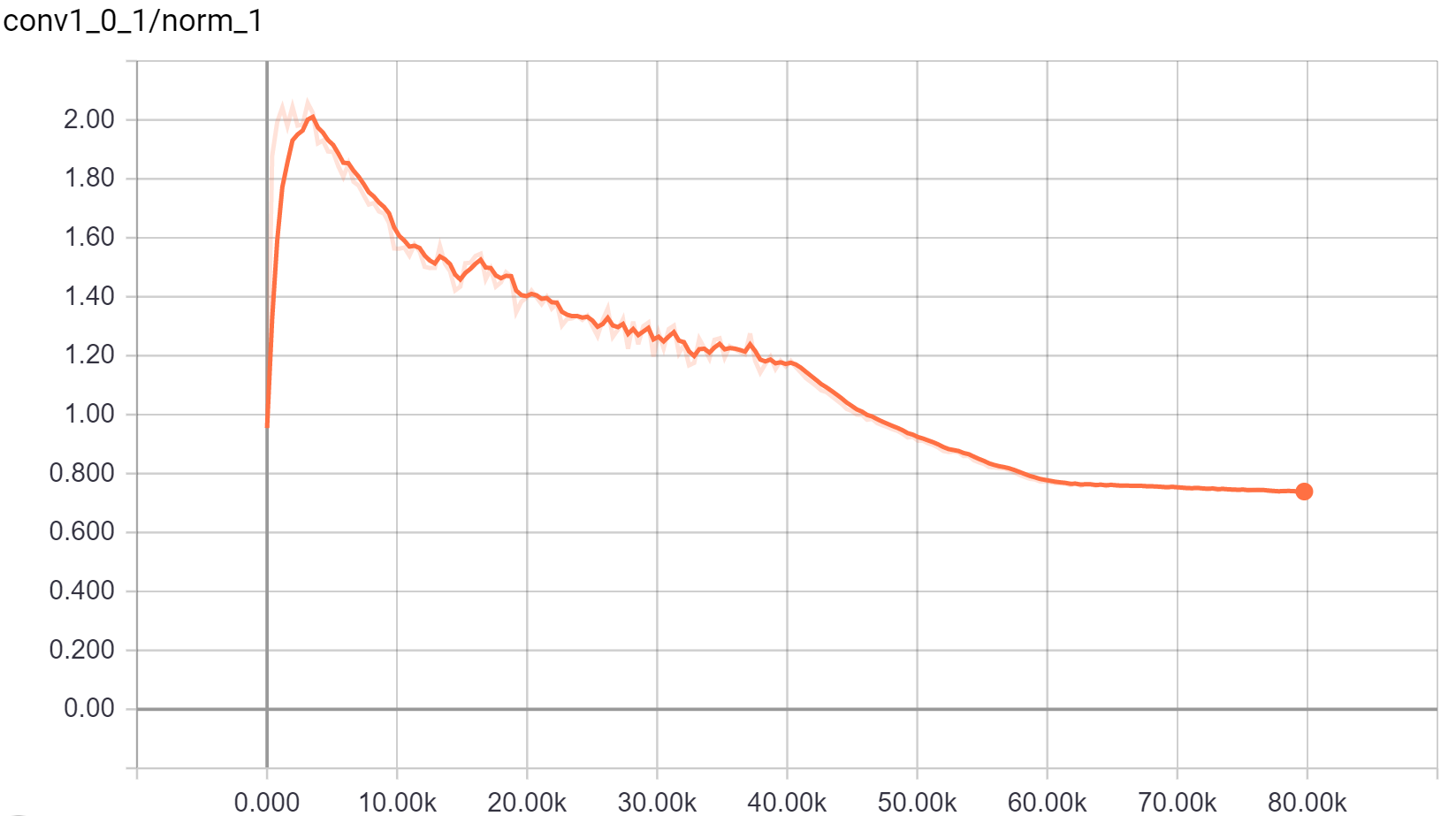}\quad\includegraphics[width=.3\linewidth]{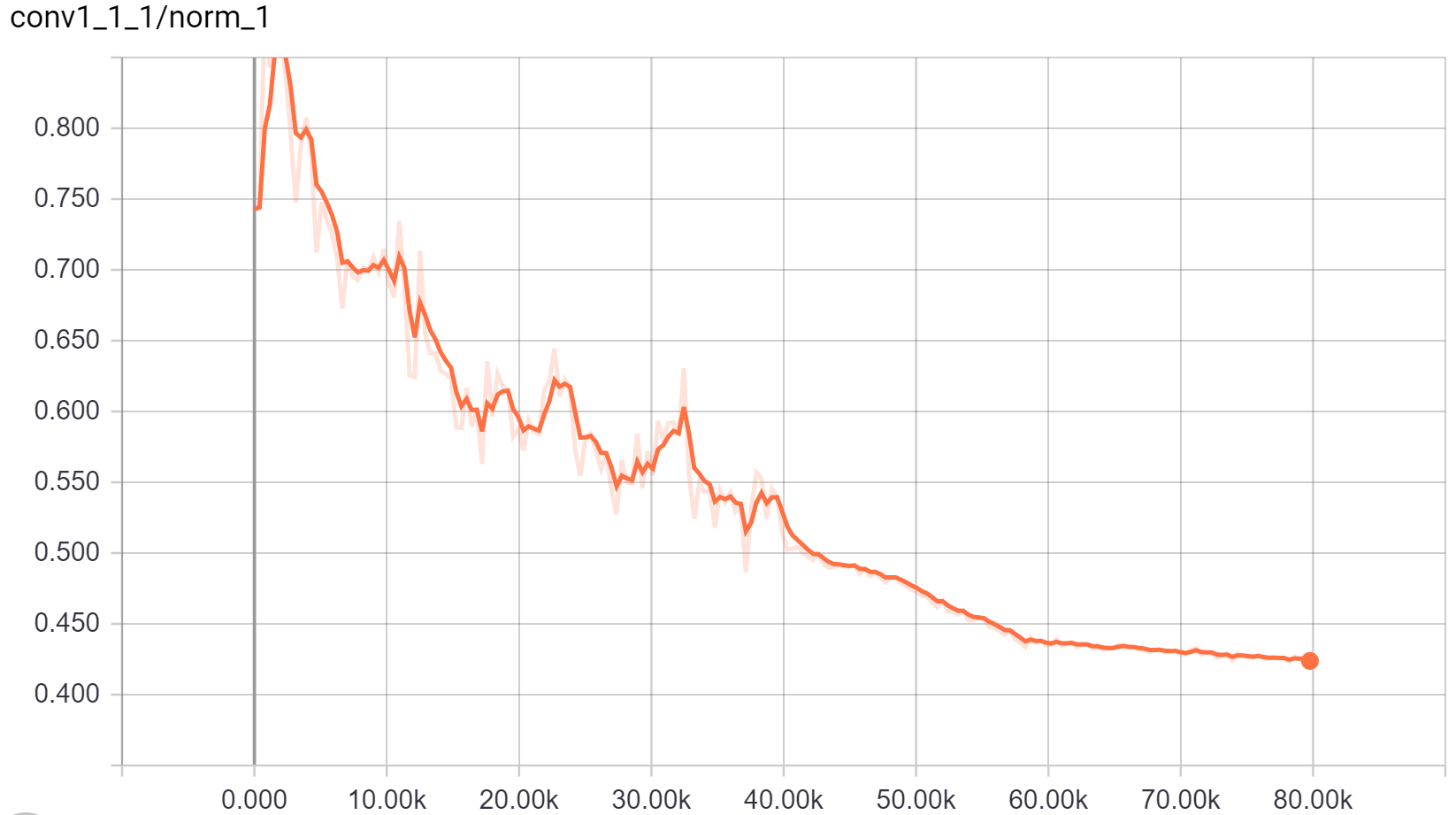}\quad\includegraphics[width=.3\linewidth]{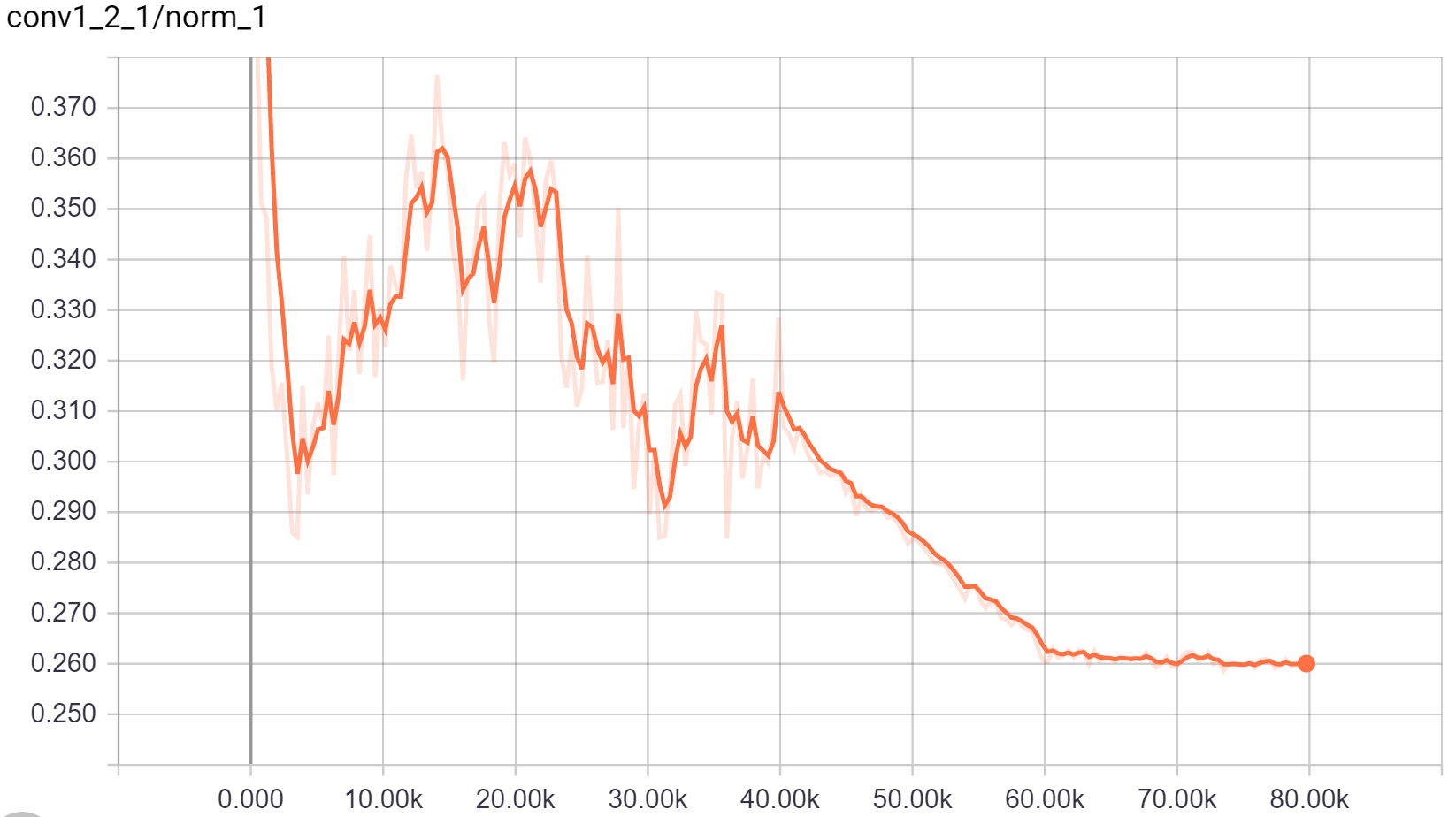}
\\[\baselineskip]% adds vertical line spacing
\includegraphics[width=.3\linewidth]{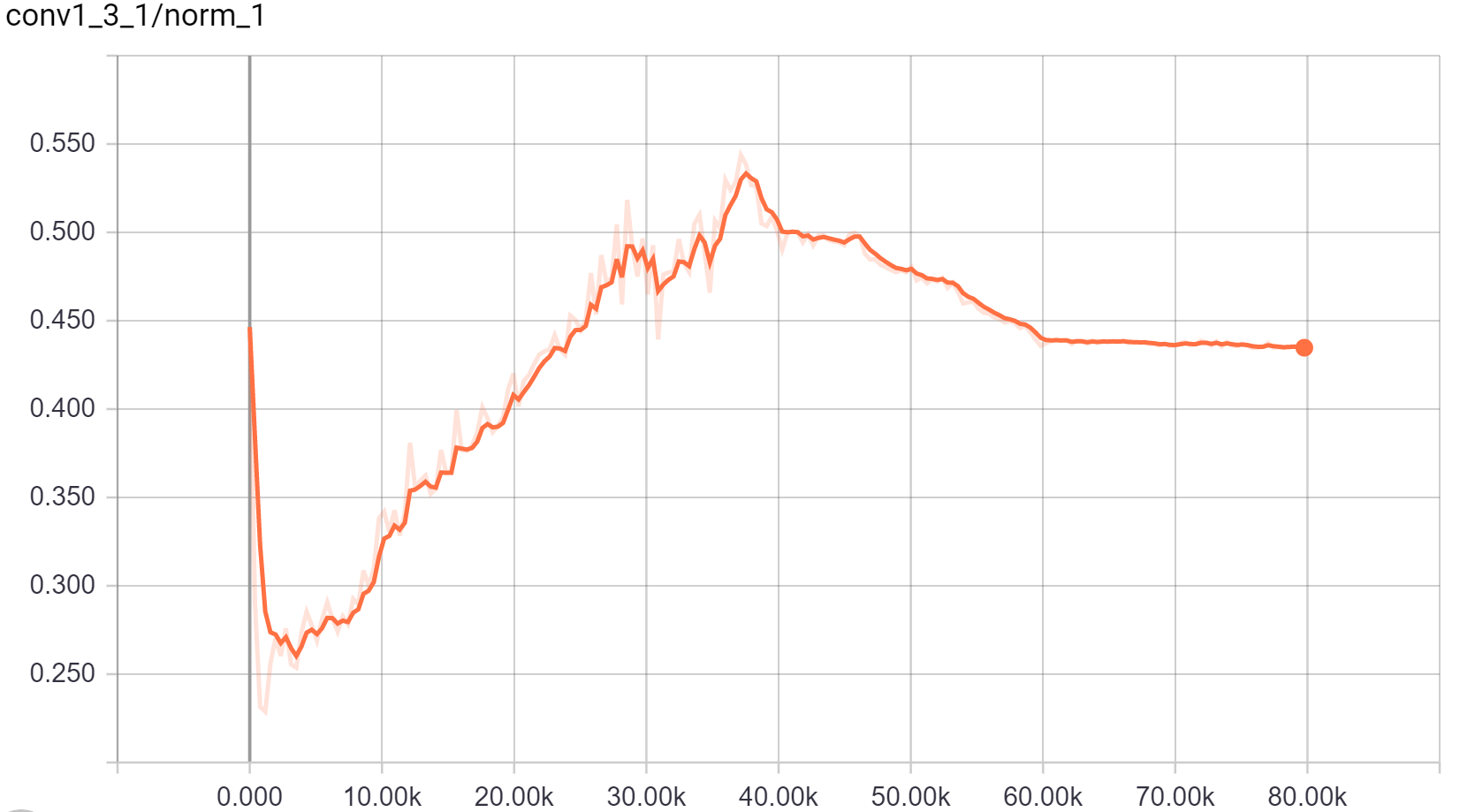}\quad\includegraphics[width=.3\linewidth]{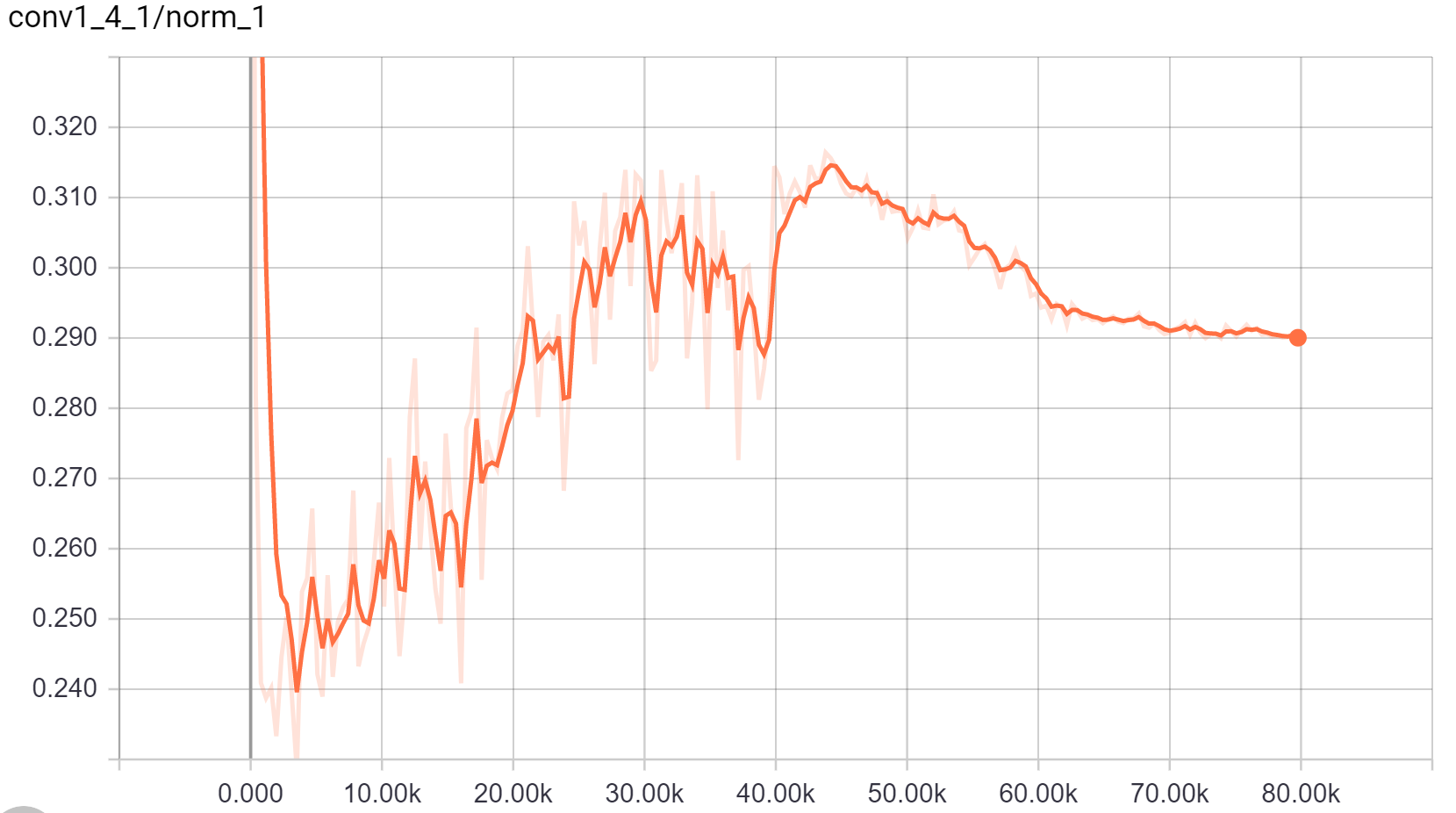}\quad\includegraphics[width=.3\linewidth]{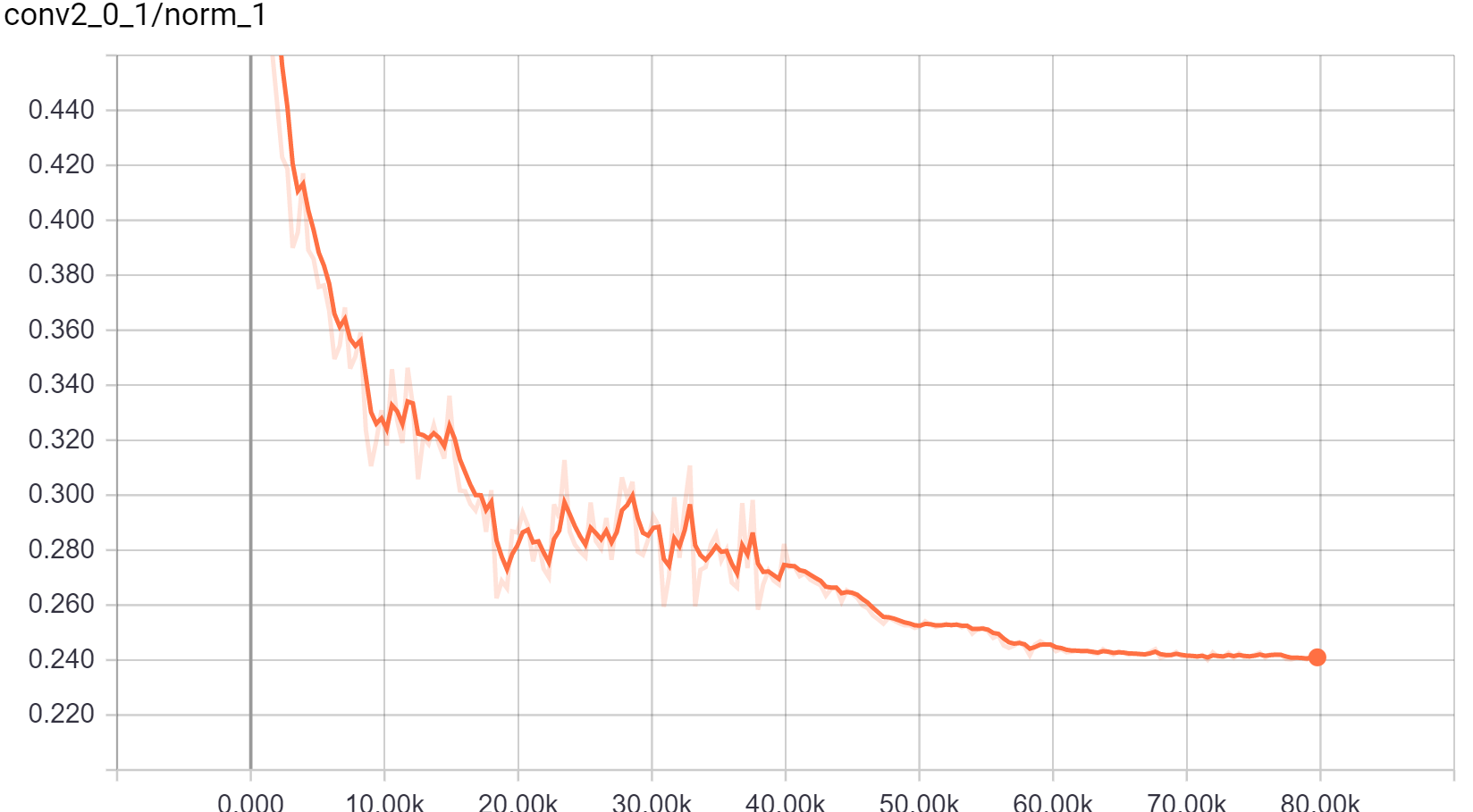}
\\[\baselineskip]% adds vertical line spacing
\includegraphics[width=.3\linewidth]{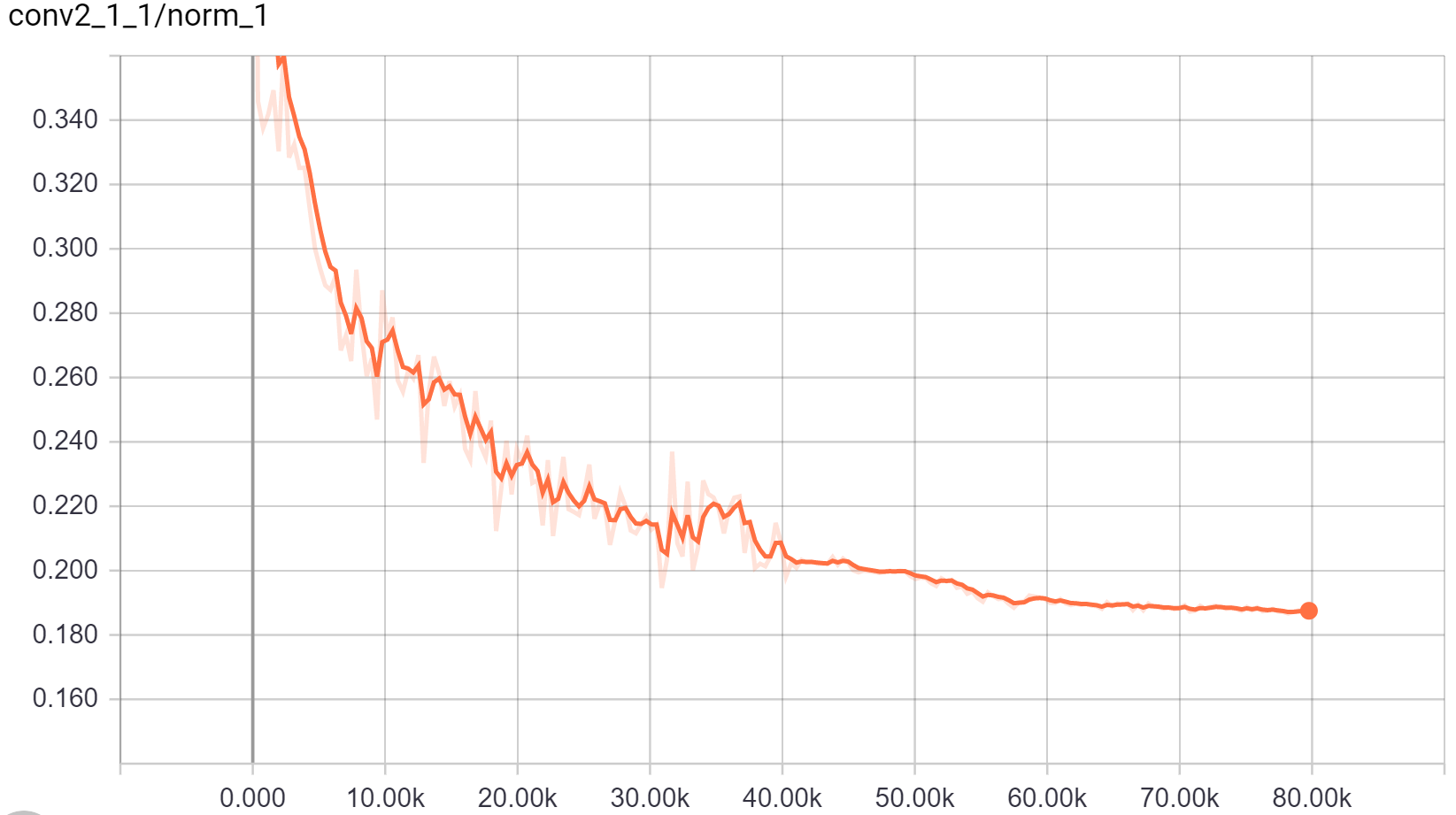}\quad\includegraphics[width=.3\linewidth]{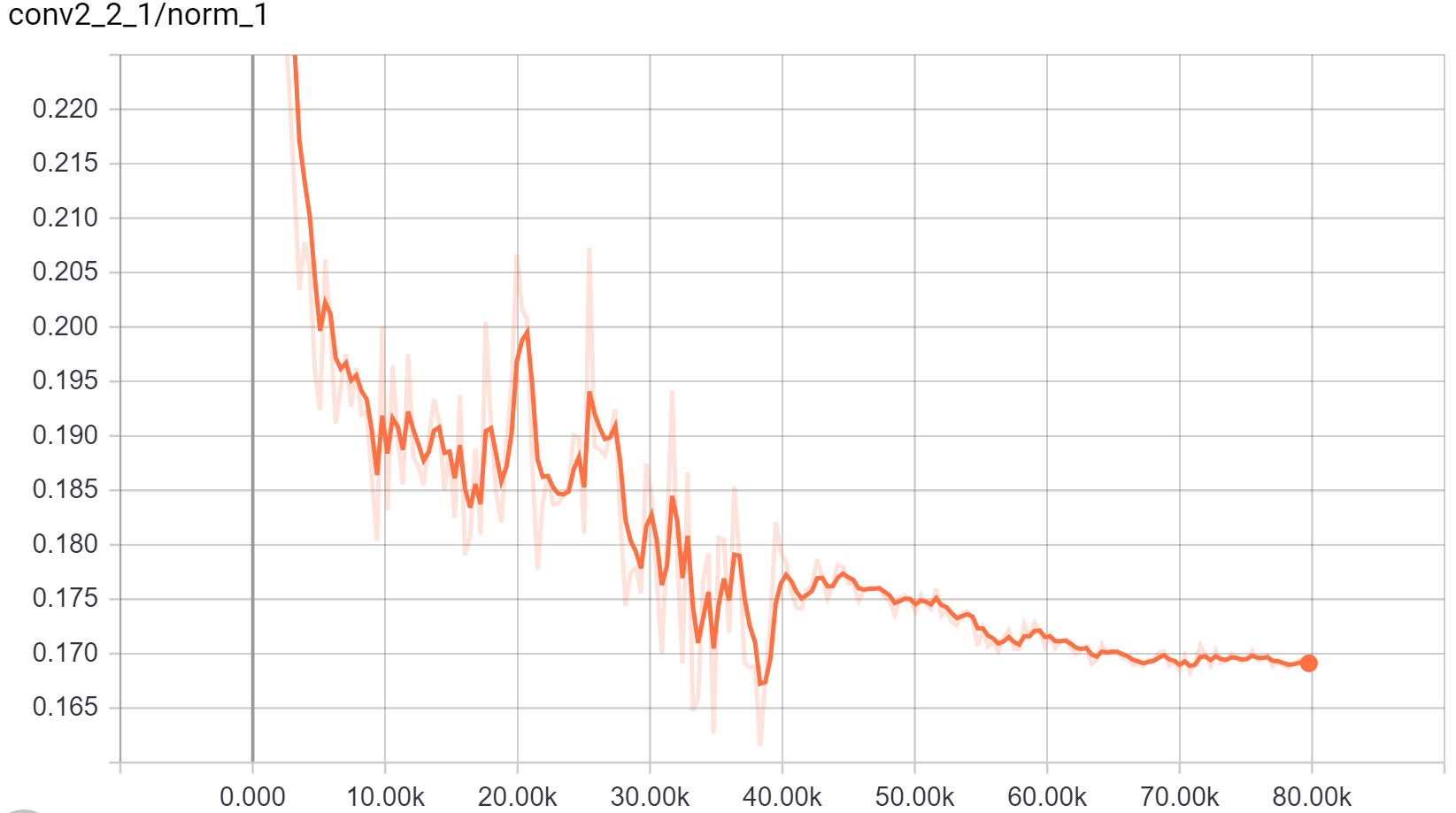}\quad\includegraphics[width=.3\linewidth]{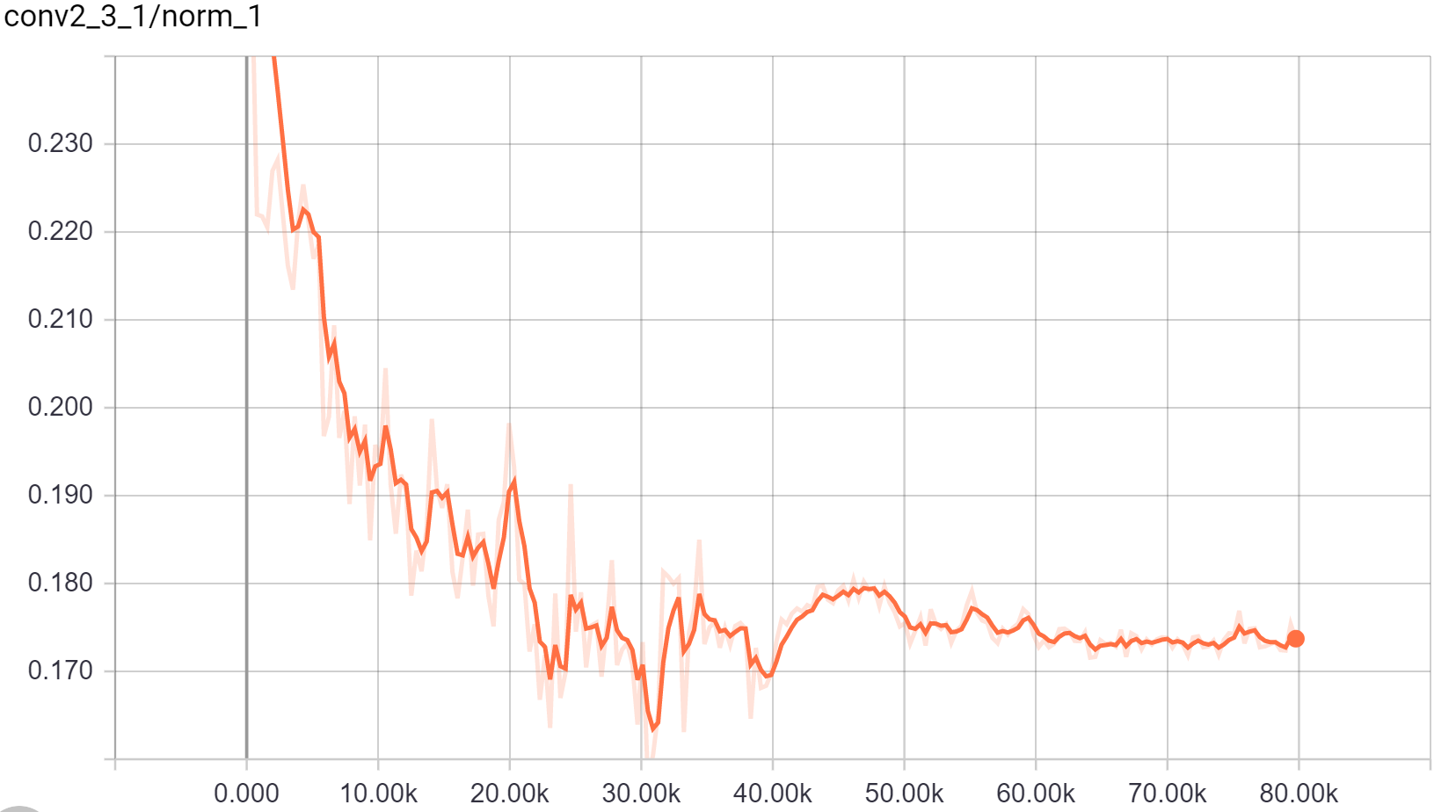}
\\[\baselineskip]% adds vertical line spacing
\includegraphics[width=.3\linewidth]{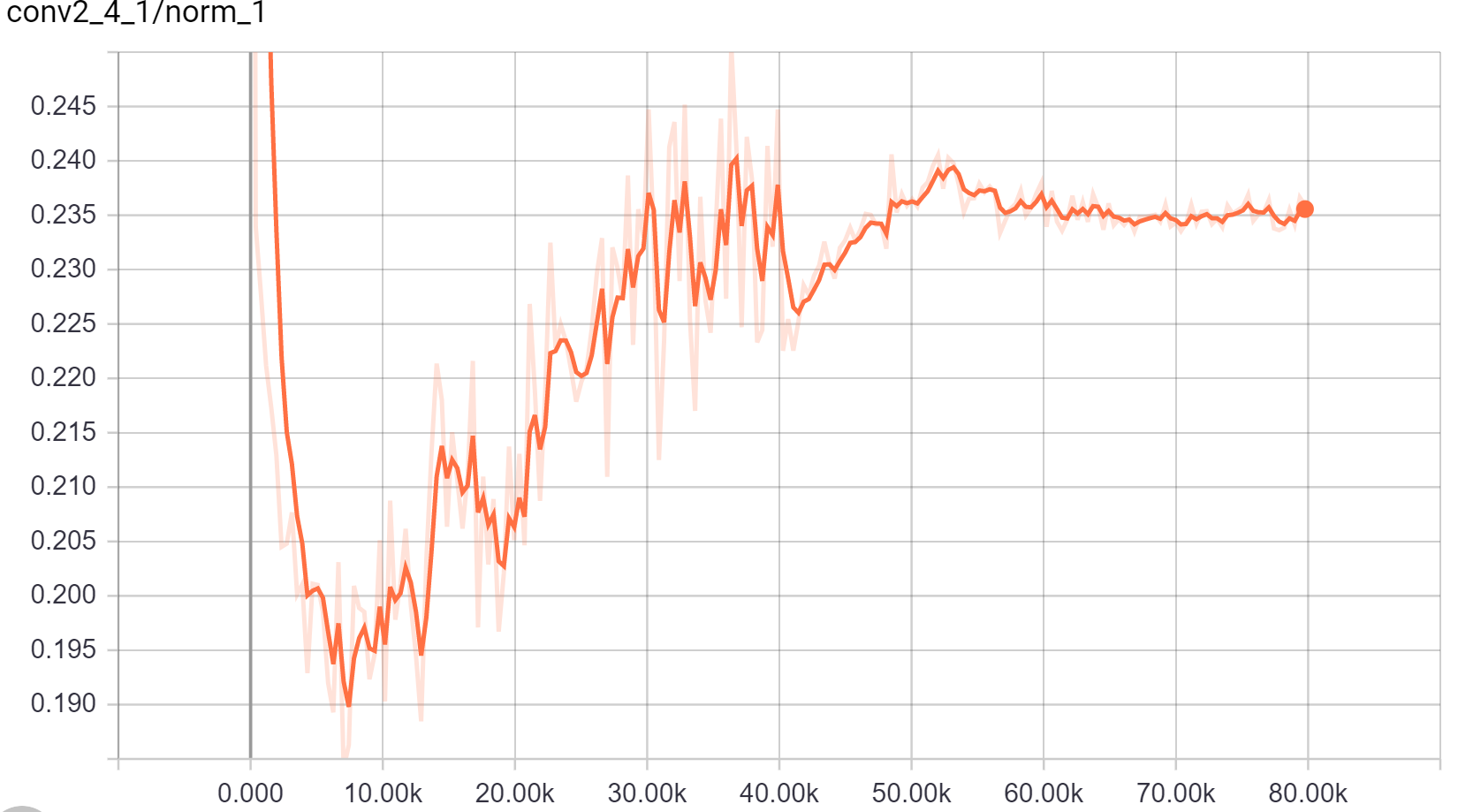}\quad\includegraphics[width=.3\linewidth]{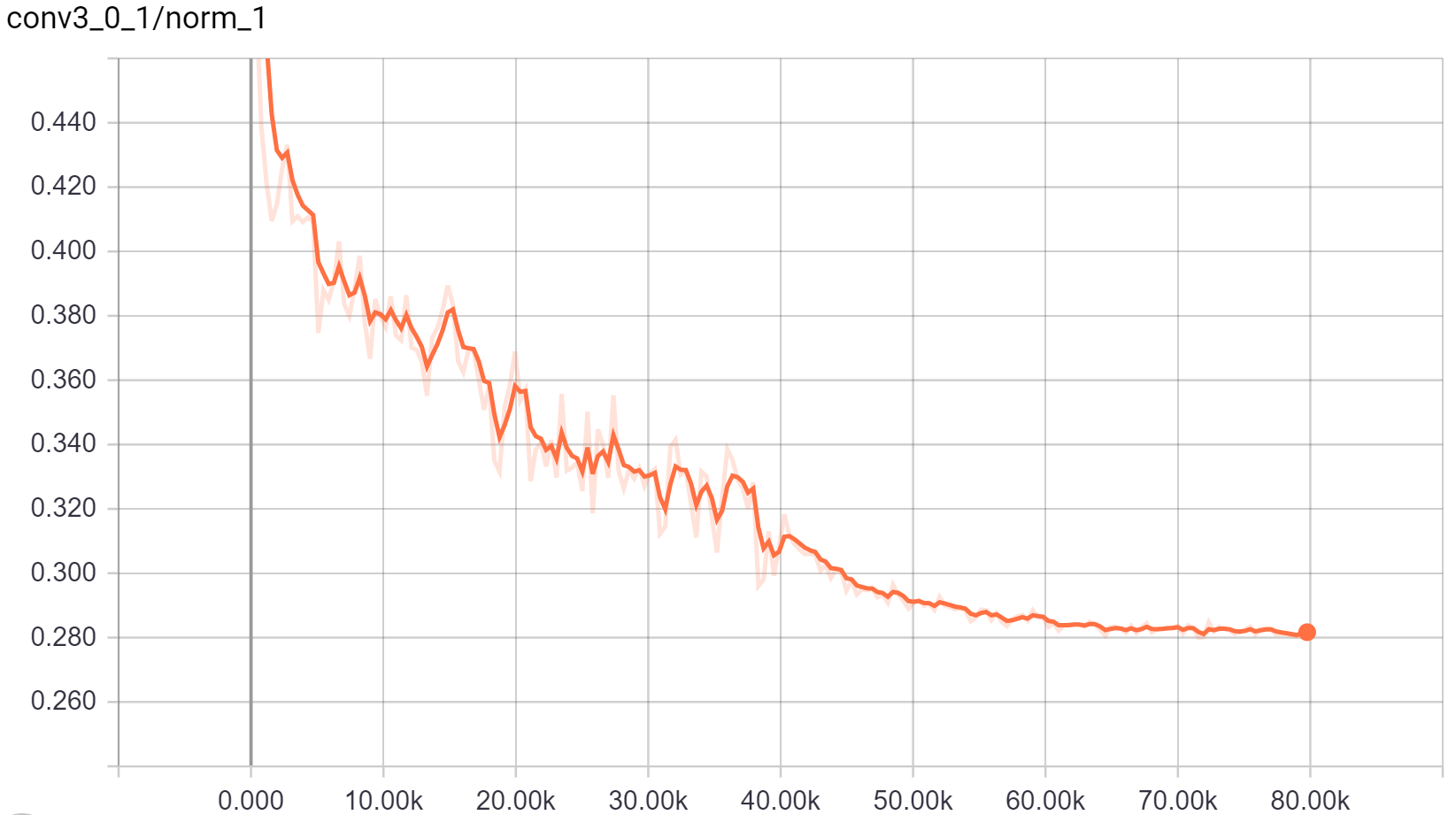}\quad\includegraphics[width=.3\linewidth]{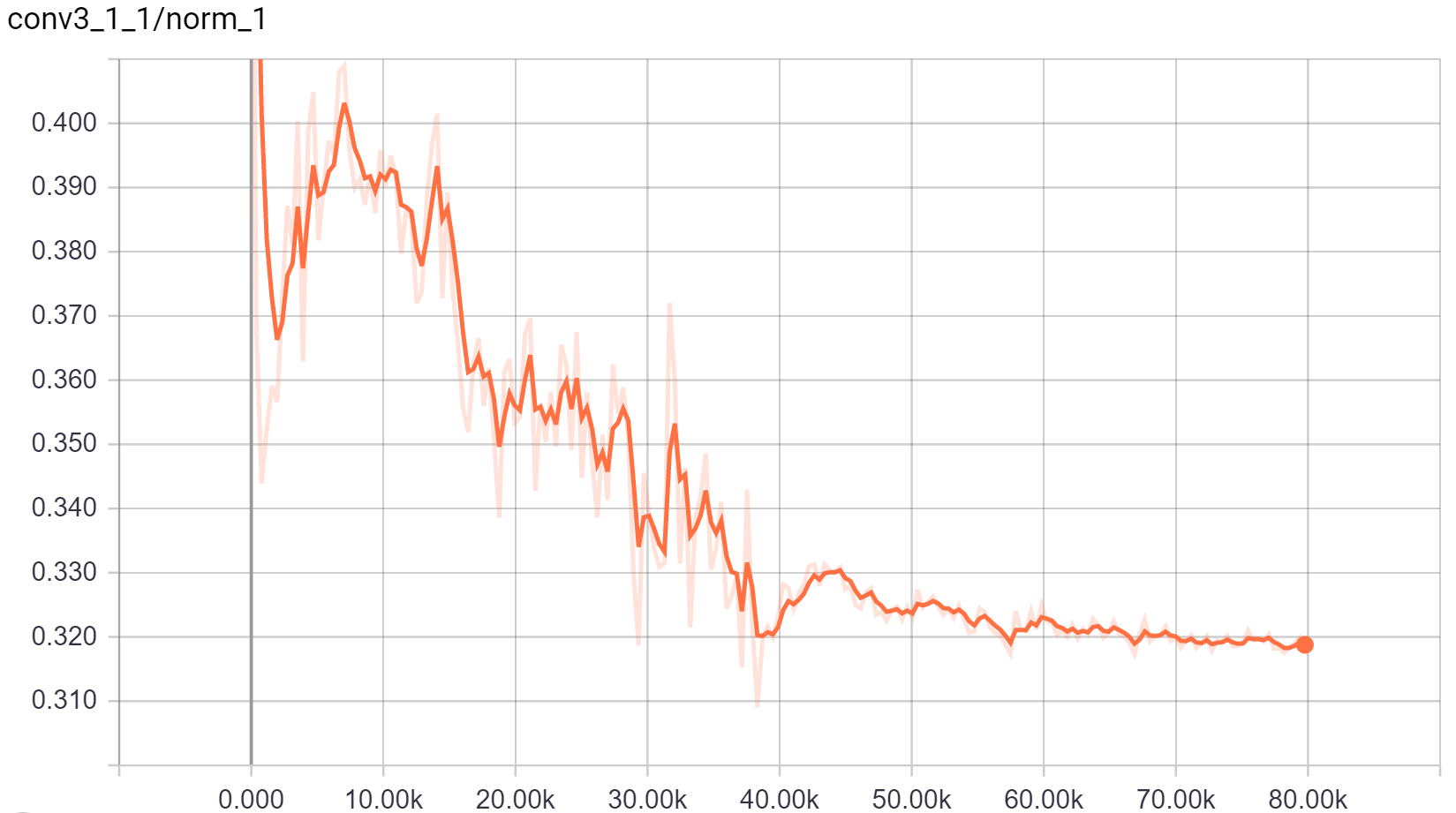}
\\[\baselineskip]% adds vertical line spacing
\includegraphics[width=.3\linewidth]{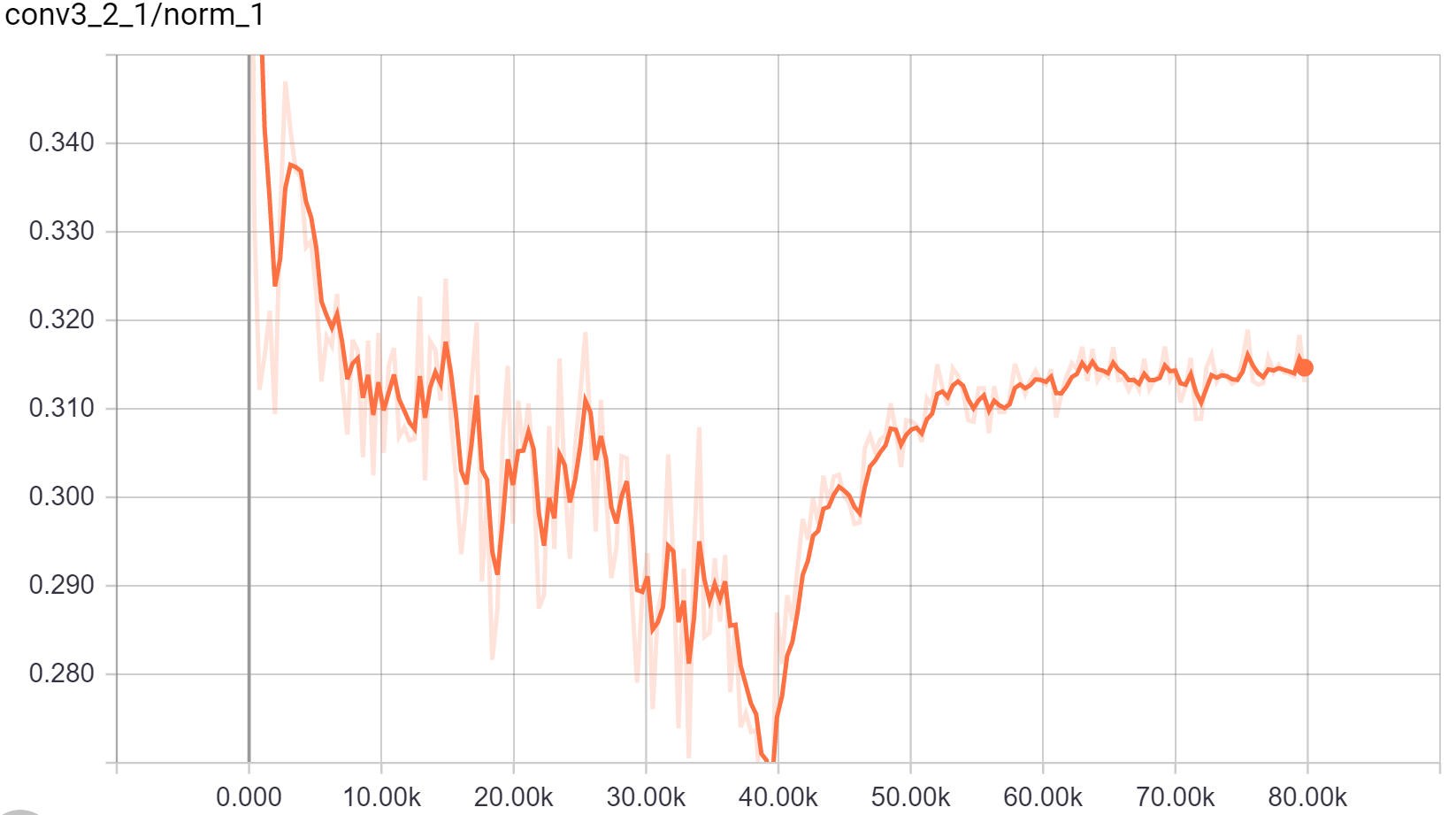}\quad\includegraphics[width=.3\linewidth]{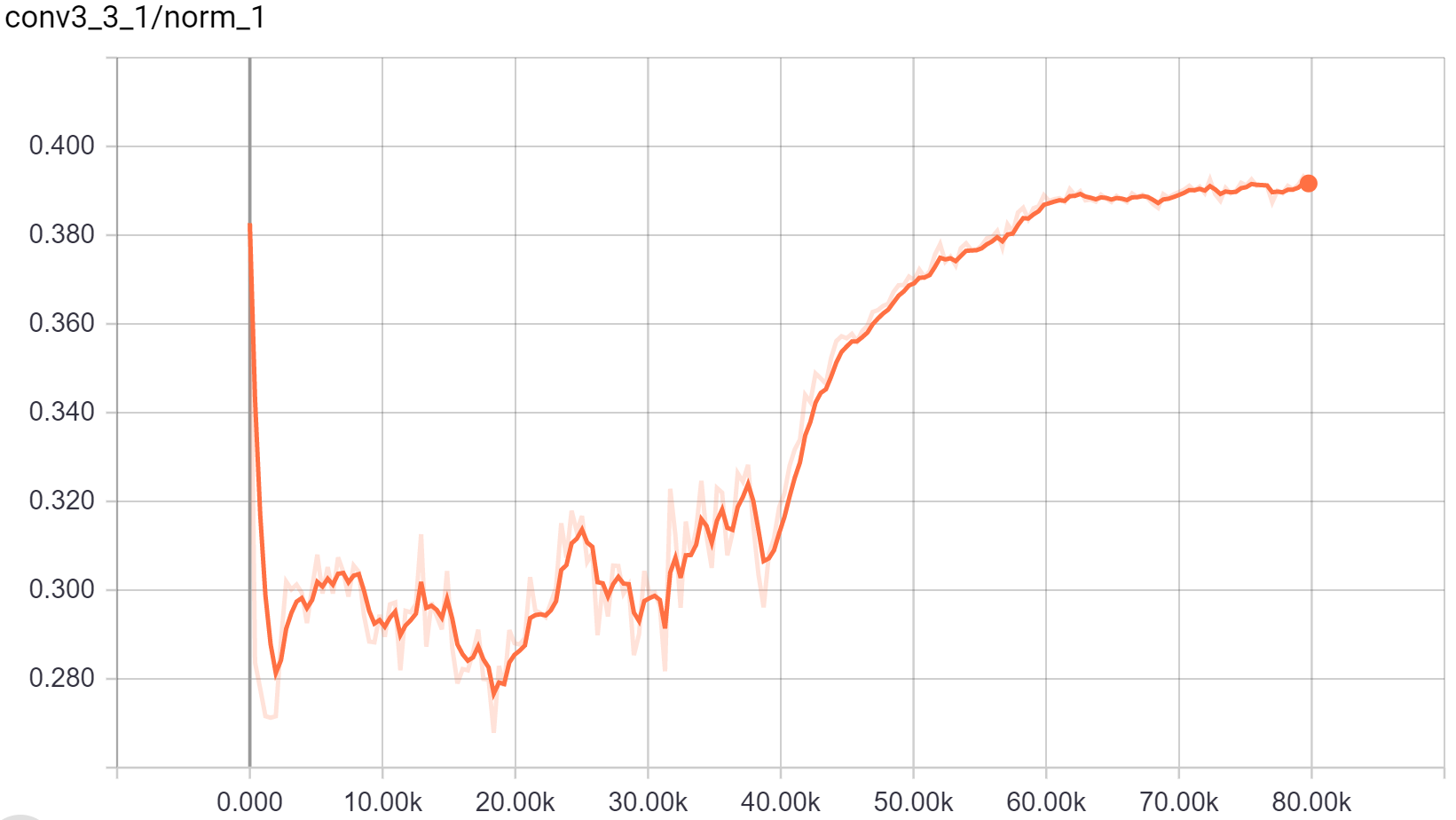}\quad\includegraphics[width=.3\linewidth]{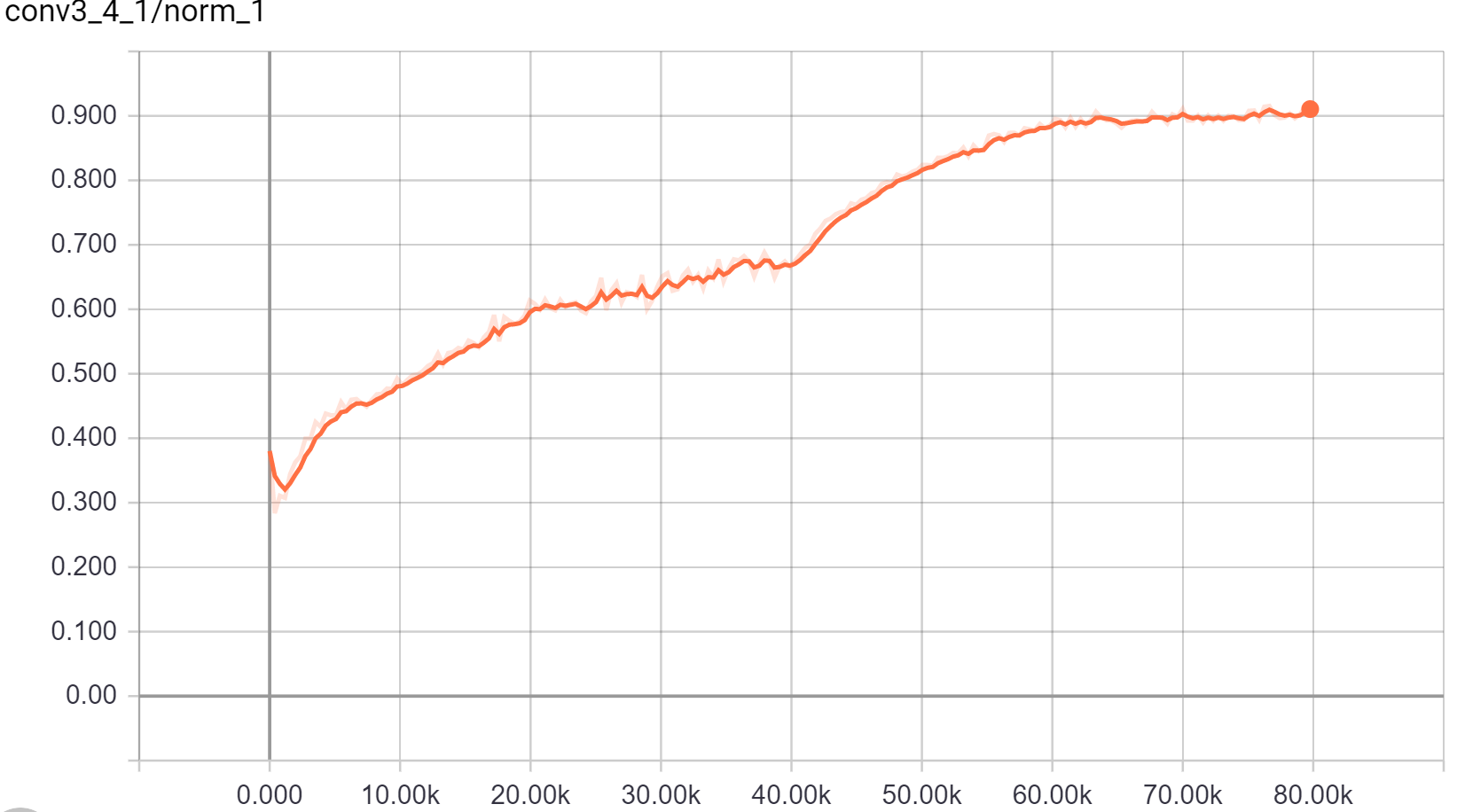}
\caption{Values of $\frac{|\mathcal{L}(x)|}{|x|}$ for each building blocks over the training steps.}
\end{center}
\end{figure}
We used tensorflow code uploaded on Github by wenxinxu, ran a 32-layer residual network on CIFAR-10 for 80k steps and computed the $\frac{|\mathcal{L}(x)|}{|x|}$ for each building blocks at each step. Here $|\cdot|$ is defined to be the square-root of the sum of squares of all entries in the tensor. We exported the output for all 15 building blocks over the training, from tensorboard. We can conclude that operator norm for every building block is smaller than 1, which meets our hypothesis.

\subsection{Advantage of ReLU over Sigmoid/Tanh}
\begin{figure}[!htbp]
\begin{center}
\includegraphics[width=.99\linewidth]{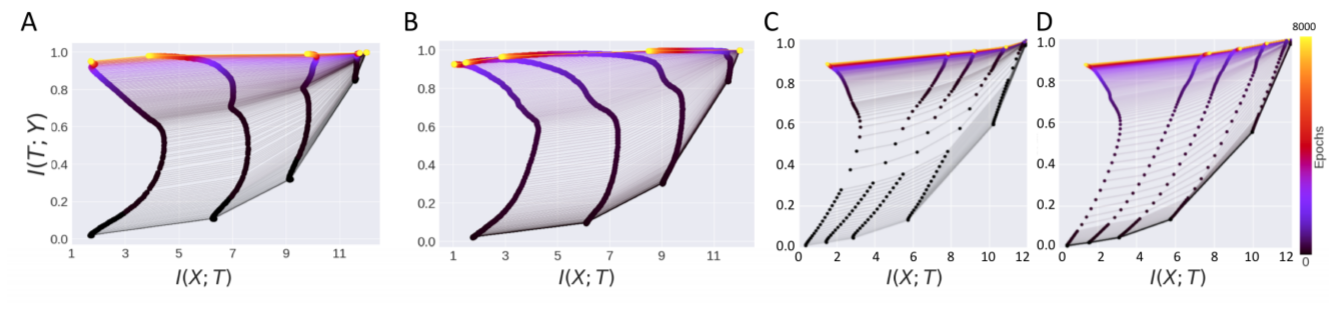}
\caption{Stochastic training and the information plane. (A) tanh network trained with SGD. (B) tanh
network trained with BGD. (C) ReLU network trained with SGD. (D) ReLU network trained with
BGD. Both random and non-random training procedures show similar information plane dynamics.}
\end{center}
\end{figure}
Here we use full reference to experimental results from \citet{michael2018on}.

Although Sigmoid and Tanh functions are mathematically invertible, they push large amount of information to the boundary of the range, which in practice will be classified as a single bin, making them highly noninvertible. On the other hand, ReLU keeps at least half of the information from input.

Their experimental result matches our theory: Tanh function compresses information and ReLU keeps fair amount of information.

We would like to emphasis again that a network only needs one compressive activation function, which is in the last layer playing the role similar to SVM.

\subsection{iRevNet}
\begin{figure}[!htbp]
\begin{center}
\includegraphics[width=.55\linewidth]{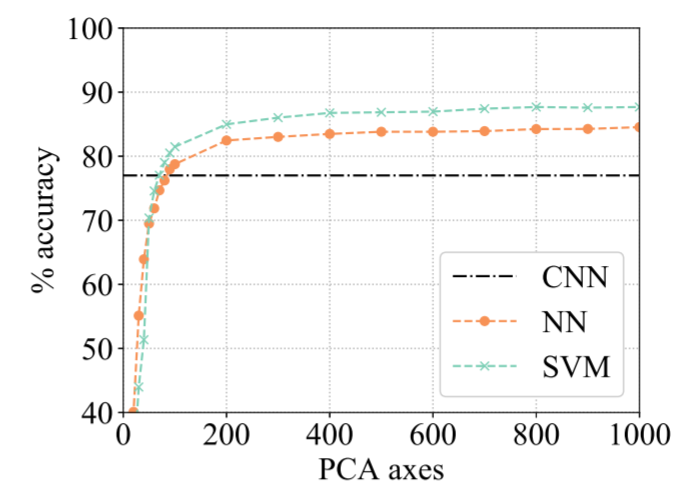}
\caption{Accuracy of a linear SVM and nearest neighbor against the number of principal components retaine.}
\end{center}
\end{figure}
Here we use full reference to experimental results from  \citet{jacobsen:hal-01712808}.

They projected the feature data of last layer of the network to a low dimensional space by PCA and did SVM on them. The good performance of SVM shows the intrinsic dimensionality of the feature data in the last layer is low, which supports our seperability assumptions.

\section{Conclusion}
In this paper, we analyzed the dynamics of the information plane proposed by \citet{shwartz2017opening}. More importantly, we gave a hypothesis for the learning structure of deep neural network, and answered the questions arised from \citet{michael2018on}.

\bibliographystyle{plainnat}
\bibliography{Info_Learning}

\begin{thebibliography}{33}
\providecommand{\natexlab}[1]{#1}
\providecommand{\url}[1]{\texttt{#1}}
\expandafter\ifx\csname urlstyle\endcsname\relax
  \providecommand{\doi}[1]{doi: #1}\else
  \providecommand{\doi}{doi: \begingroup \urlstyle{rm}\Url}\fi

\bibitem[Achille and Soatto(2016)]{journals/corr/AchilleS16}
Alessandro Achille and Stefano Soatto.
\newblock Information dropout: learning optimal representations through noise.
\newblock \emph{CoRR}, abs/1611.01353, 2016.
\newblock URL
  \url{http://dblp.uni-trier.de/db/journals/corr/corr1611.html#AchilleS16}.

\bibitem[Alemi et~al.(2016)Alemi, Fischer, Dillon, and
  Murphy]{DBLP:journals/corr/AlemiFD016}
Alexander~A. Alemi, Ian Fischer, Joshua~V. Dillon, and Kevin Murphy.
\newblock Deep variational information bottleneck.
\newblock \emph{CoRR}, abs/1612.00410, 2016.
\newblock URL \url{http://arxiv.org/abs/1612.00410}.

\bibitem[Bartlett et~al.(2017)Bartlett, Foster, and
  Telgarsky]{DBLP:journals/corr/BartlettFT17}
Peter~L. Bartlett, Dylan~J. Foster, and Matus Telgarsky.
\newblock Spectrally-normalized margin bounds for neural networks.
\newblock \emph{CoRR}, abs/1706.08498, 2017.
\newblock URL \url{http://arxiv.org/abs/1706.08498}.

\bibitem[Bottou(1998)]{Bottou:1999:OLS:304710.304720}
L{\'e}on Bottou.
\newblock On-line learning in neural networks.
\newblock chapter On-line Learning and Stochastic Approximations, pages 9--42.
  Cambridge University Press, New York, NY, USA, 1998.
\newblock ISBN 0-521-65263-4.
\newblock URL \url{http://dl.acm.org/citation.cfm?id=304710.304720}.

\bibitem[Chaudhari et~al.(2016)Chaudhari, Choromanska, Soatto, LeCun, Baldassi,
  Borgs, Chayes, Sagun, and Zecchina]{DBLP:journals/corr/ChaudhariCSL16}
Pratik Chaudhari, Anna Choromanska, Stefano Soatto, Yann LeCun, Carlo Baldassi,
  Christian Borgs, Jennifer~T. Chayes, Levent Sagun, and Riccardo Zecchina.
\newblock Entropy-sgd: Biasing gradient descent into wide valleys.
\newblock \emph{CoRR}, abs/1611.01838, 2016.
\newblock URL \url{http://arxiv.org/abs/1611.01838}.

\bibitem[Dauphin et~al.(2014)Dauphin, Pascanu, G{\"{u}}l{\c{c}}ehre, Cho,
  Ganguli, and Bengio]{DBLP:journals/corr/DauphinPGCGB14}
Yann Dauphin, Razvan Pascanu, {\c{C}}aglar G{\"{u}}l{\c{c}}ehre, Kyunghyun Cho,
  Surya Ganguli, and Yoshua Bengio.
\newblock Identifying and attacking the saddle point problem in
  high-dimensional non-convex optimization.
\newblock \emph{CoRR}, abs/1406.2572, 2014.
\newblock URL \url{http://arxiv.org/abs/1406.2572}.

\bibitem[Dinh et~al.(2017)Dinh, Pascanu, Bengio, and
  Bengio]{DBLP:journals/corr/DinhPBB17}
Laurent Dinh, Razvan Pascanu, Samy Bengio, and Yoshua Bengio.
\newblock Sharp minima can generalize for deep nets.
\newblock \emph{CoRR}, abs/1703.04933, 2017.
\newblock URL \url{http://arxiv.org/abs/1703.04933}.

\bibitem[He et~al.(2015)He, Zhang, Ren, and Sun]{DBLP:journals/corr/HeZRS15}
Kaiming He, Xiangyu Zhang, Shaoqing Ren, and Jian Sun.
\newblock Deep residual learning for image recognition.
\newblock \emph{CoRR}, abs/1512.03385, 2015.
\newblock URL \url{http://arxiv.org/abs/1512.03385}.

\bibitem[He et~al.(2016)He, Zhang, Ren, and Sun]{DBLP:journals/corr/HeZR016}
Kaiming He, Xiangyu Zhang, Shaoqing Ren, and Jian Sun.
\newblock Identity mappings in deep residual networks.
\newblock \emph{CoRR}, abs/1603.05027, 2016.
\newblock URL \url{http://arxiv.org/abs/1603.05027}.

\bibitem[Ioffe and Szegedy(2015)]{DBLP:journals/corr/IoffeS15}
Sergey Ioffe and Christian Szegedy.
\newblock Batch normalization: Accelerating deep network training by reducing
  internal covariate shift.
\newblock \emph{CoRR}, abs/1502.03167, 2015.
\newblock URL \url{http://arxiv.org/abs/1502.03167}.

\bibitem[Jacobsen et~al.(2018)Jacobsen, Smeulders, and
  Oyallon]{jacobsen:hal-01712808}
J{\"o}rn-Henrik Jacobsen, Arnold Smeulders, and Edouard Oyallon.
\newblock {i-RevNet: Deep Invertible Networks}.
\newblock In \emph{{ICLR 2018 - International Conference on Learning
  Representations}}, Vancouver, Canada, April 2018.
\newblock URL \url{https://hal.archives-ouvertes.fr/hal-01712808}.

\bibitem[Kraskov et~al.(2004)Kraskov, St\"ogbauer, and
  Grassberger]{PhysRevE.69.066138}
Alexander Kraskov, Harald St\"ogbauer, and Peter Grassberger.
\newblock Estimating mutual information.
\newblock \emph{Phys. Rev. E}, 69:\penalty0 066138, Jun 2004.
\newblock \doi{10.1103/PhysRevE.69.066138}.
\newblock URL \url{https://link.aps.org/doi/10.1103/PhysRevE.69.066138}.

\bibitem[Krizhevsky et~al.(2012)Krizhevsky, Sutskever, and
  Hinton]{NIPS2012_4824}
Alex Krizhevsky, Ilya Sutskever, and Geoffrey~E Hinton.
\newblock Imagenet classification with deep convolutional neural networks.
\newblock In F.~Pereira, C.~J.~C. Burges, L.~Bottou, and K.~Q. Weinberger,
  editors, \emph{Advances in Neural Information Processing Systems 25}, pages
  1097--1105. Curran Associates, Inc., 2012.
\newblock URL
  \url{http://papers.nips.cc/paper/4824-imagenet-classification-with-deep-convolutional-neural-networks.pdf}.

\bibitem[Lax(2002)]{lax2002functional}
P.D. Lax.
\newblock \emph{Functional analysis}.
\newblock Pure and applied mathematics. Wiley, 2002.
\newblock ISBN 9780471556046.
\newblock URL \url{https://books.google.com/books?id=-jbvAAAAMAAJ}.

\bibitem[Nair and Hinton(2010)]{Nair:2010:RLU:3104322.3104425}
Vinod Nair and Geoffrey~E. Hinton.
\newblock Rectified linear units improve restricted boltzmann machines.
\newblock In \emph{Proceedings of the 27th International Conference on
  International Conference on Machine Learning}, ICML'10, pages 807--814, USA,
  2010. Omnipress.
\newblock ISBN 978-1-60558-907-7.
\newblock URL \url{http://dl.acm.org/citation.cfm?id=3104322.3104425}.

\bibitem[Neyshabur et~al.(2017{\natexlab{a}})Neyshabur, Bhojanapalli,
  McAllester, and Srebro]{DBLP:journals/corr/NeyshaburBMS17}
Behnam Neyshabur, Srinadh Bhojanapalli, David McAllester, and Nathan Srebro.
\newblock Exploring generalization in deep learning.
\newblock \emph{CoRR}, abs/1706.08947, 2017{\natexlab{a}}.
\newblock URL \url{http://arxiv.org/abs/1706.08947}.

\bibitem[Neyshabur et~al.(2017{\natexlab{b}})Neyshabur, Bhojanapalli,
  McAllester, and Srebro]{Neyshabur2017ExploringGI}
Behnam Neyshabur, Srinadh Bhojanapalli, David McAllester, and Nathan Srebro.
\newblock Exploring generalization in deep learning.
\newblock In \emph{NIPS}, 2017{\natexlab{b}}.

\bibitem[Poggio et~al.(2017)Poggio, Liao, Miranda, Rosasco, Boix, Hidary, and
  Mhaskar]{3266}
Tomaso Poggio, Qianli Liao, Brando Miranda, Lorenzo Rosasco, Xavier Boix, Jack
  Hidary, and Hrushikesh Mhaskar.
\newblock Theory of deep learning iii: explaining the non-overfitting puzzle.
\newblock 12/2017 2017.

\bibitem[Raginsky et~al.(2017)Raginsky, Rakhlin, and
  Telgarsky]{pmlr-v65-raginsky17a}
Maxim Raginsky, Alexander Rakhlin, and Matus Telgarsky.
\newblock Non-convex learning via stochastic gradient langevin dynamics: a
  nonasymptotic analysis.
\newblock In Satyen Kale and Ohad Shamir, editors, \emph{Proceedings of the
  2017 Conference on Learning Theory}, volume~65 of \emph{Proceedings of
  Machine Learning Research}, pages 1674--1703, Amsterdam, Netherlands, 07--10
  Jul 2017. PMLR.
\newblock URL \url{http://proceedings.mlr.press/v65/raginsky17a.html}.

\bibitem[Roberts and Tweedie(1996)]{roberts1996}
Gareth~O. Roberts and Richard~L. Tweedie.
\newblock Exponential convergence of langevin distributions and their discrete
  approximations.
\newblock \emph{Bernoulli}, 2\penalty0 (4):\penalty0 341--363, 12 1996.
\newblock URL \url{https://projecteuclid.org:443/euclid.bj/1178291835}.

\bibitem[Sagun et~al.(2017)Sagun, Evci, G{\"{u}}ney, Dauphin, and
  Bottou]{DBLP:journals/corr/SagunEGDB17}
Levent Sagun, Utku Evci, V.~Ugur G{\"{u}}ney, Yann Dauphin, and L{\'{e}}on
  Bottou.
\newblock Empirical analysis of the hessian of over-parametrized neural
  networks.
\newblock \emph{CoRR}, abs/1706.04454, 2017.
\newblock URL \url{http://arxiv.org/abs/1706.04454}.

\bibitem[Saxe et~al.(2018)Saxe, Bansal, Dapello, Advani, Kolchinsky, Tracey,
  and Cox]{michael2018on}
Andrew~Michael Saxe, Yamini Bansal, Joel Dapello, Madhu Advani, Artemy
  Kolchinsky, Brendan~Daniel Tracey, and David~Daniel Cox.
\newblock On the information bottleneck theory of deep learning.
\newblock In \emph{International Conference on Learning Representations}, 2018.
\newblock URL \url{https://openreview.net/forum?id=ry_WPG-A-}.

\bibitem[Shamir et~al.(2010)Shamir, Sabato, and Tishby]{SHAMIR20102696}
Ohad Shamir, Sivan Sabato, and Naftali Tishby.
\newblock Learning and generalization with the information bottleneck.
\newblock \emph{Theoretical Computer Science}, 411\penalty0 (29):\penalty0 2696
  -- 2711, 2010.
\newblock ISSN 0304-3975.
\newblock \doi{https://doi.org/10.1016/j.tcs.2010.04.006}.
\newblock URL
  \url{http://www.sciencedirect.com/science/article/pii/S030439751000201X}.
\newblock Algorithmic Learning Theory (ALT 2008).

\bibitem[Shwartz-Ziv and Tishby(2017)]{shwartz2017opening}
Ravid Shwartz-Ziv and Naftali Tishby.
\newblock Opening the black box of deep neural networks via information.
\newblock \emph{arXiv preprint arXiv:1703.00810}, 2017.

\bibitem[Soudry et~al.(2018)Soudry, Hoffer, and Srebro]{soudry2018the}
Daniel Soudry, Elad Hoffer, and Nathan Srebro.
\newblock The implicit bias of gradient descent on separable data.
\newblock In \emph{International Conference on Learning Representations}, 2018.
\newblock URL \url{https://openreview.net/forum?id=r1q7n9gAb}.

\bibitem[Srivastava et~al.(2014)Srivastava, Hinton, Krizhevsky, Sutskever, and
  Salakhutdinov]{JMLR:v15:srivastava14a}
Nitish Srivastava, Geoffrey Hinton, Alex Krizhevsky, Ilya Sutskever, and Ruslan
  Salakhutdinov.
\newblock Dropout: A simple way to prevent neural networks from overfitting.
\newblock \emph{Journal of Machine Learning Research}, 15:\penalty0 1929--1958,
  2014.
\newblock URL \url{http://jmlr.org/papers/v15/srivastava14a.html}.

\bibitem[Troost et~al.(2018)Troost, Seeliger, and van Gerven]{article}
Marjolein Troost, Katja Seeliger, and Marcel van Gerven.
\newblock Generalization of an upper bound on the number of nodes needed to
  achieve linear separability.
\newblock 02 2018.

\bibitem[Tzen et~al.(2018)Tzen, Liang, and
  Raginsky]{DBLP:journals/corr/abs-1802-06439}
Belinda Tzen, Tengyuan Liang, and Maxim Raginsky.
\newblock Local optimality and generalization guarantees for the langevin
  algorithm via empirical metastability.
\newblock \emph{CoRR}, abs/1802.06439, 2018.
\newblock URL \url{http://arxiv.org/abs/1802.06439}.

\bibitem[Vapnik(1995)]{Vapnik:1995:NSL:211359}
Vladimir~N. Vapnik.
\newblock \emph{The Nature of Statistical Learning Theory}.
\newblock Springer-Verlag New York, Inc., New York, NY, USA, 1995.
\newblock ISBN 0-387-94559-8.

\bibitem[Vershynin(2018)]{vershynin_2018}
Roman Vershynin.
\newblock \emph{High-Dimensional Probability: An Introduction with Applications
  in Data Science}.
\newblock Cambridge Series in Statistical and Probabilistic Mathematics.
  Cambridge University Press, 2018.

\bibitem[Zhang et~al.(2017{\natexlab{a}})Zhang, Bengio, Hardt, Recht, and
  Vinyals]{45820}
Chiyuan Zhang, Samy Bengio, Moritz Hardt, Benjamin Recht, and Oriol Vinyals.
\newblock Understanding deep learning requires rethinking generalization.
\newblock 2017{\natexlab{a}}.
\newblock URL \url{https://arxiv.org/abs/1611.03530}.

\bibitem[Zhang et~al.(2017{\natexlab{b}})Zhang, Liang, and
  Charikar]{Zhang2017AHT}
Yuchen Zhang, Percy Liang, and Moses Charikar.
\newblock A hitting time analysis of stochastic gradient langevin dynamics.
\newblock In \emph{COLT}, 2017{\natexlab{b}}.

\bibitem[Zheng et~al.(2018)Zheng, Sang, and Xu]{zheng2018understanding}
Guanhua Zheng, Jitao Sang, and Changsheng Xu.
\newblock Understanding deep learning generalization by maximum entropy, 2018.
\newblock URL \url{https://openreview.net/forum?id=r1kj4ACp-}.

\end{thebibliography}

\appendix

\section{Bound for $\I(X;\tilde{Y})$}
\label{proof(1)}

The intuition is Jensen's inequality is loose if the term in $\log$ deviates from constant by a lot.

Consider the second quantity $\I(X;\tilde{Y})$ given by:
\begin{equation}
\I(X;\tilde{Y})=\int_{\tilde{Y}} \int_{X}p(x,\tilde{y}) \log(\frac{p(x,\tilde{y})}{p(x)p(\tilde{y})})dxd\tilde{y}
\end{equation}
Our goal is to show that (1) is decreasing at a rate upper bounded by polynomial.

Denote $f(x,\tilde{y}) = \frac{p(x)p(\tilde{y})}{p(x,\tilde{y})}$, consider the 2nd order Taylor form around 1:
\begin{equation}
\log(f(x,\tilde{y})) = \log(1)+(f(x,\tilde{y})-1)-\frac{1}{2c^2}(f(x,\tilde{y})-1)^2,
\end{equation}
where $c$ is between $f(x,\tilde{y})$ and 1.

Observing $\E(f(x,\tilde{y}))=1$, then substitute (2) into (1) get:
\begin{equation}
\I(X; \tilde{Y})=\int_{\tilde{Y}} \int_{X} p(x,\tilde{y}) \frac{1}{2c^2}(f(x,\tilde{y})-1)^2dxd\tilde{y}.
\end{equation}

If $c \geq 1$, then $\frac{1}{2c^2}(f(x,\tilde{y})-1)^2 \leq \frac{1}{2}(f(x,\tilde{y})-1)^2$; if $c \leq 1$, then $\frac{1}{2c^2}(f(x,\tilde{y})-1)^2 \leq \frac{1}{2f(x,\tilde{y})^2}(f(x,\tilde{y})-1)^2=\frac{1}{2}(1-\frac{1}{f(x,\tilde{y})})^2$.

It follows we can upper bound (3) by:
\begin{equation}
\begin{array} {lcl}
\I(X; \tilde{Y}) & \leq & \frac{1}{2}\int_{\tilde{Y}} \int_{X} p(x,\tilde{y}) [(f(x,\tilde{y})-1)^2 + (1-\frac{1}{f(x,\tilde{y})})^2]dxd\tilde{y} \\
& = & \Var(f(X,\tilde{Y}))+\Var(\frac{1}{f(X,\tilde{Y})})\\
& = & \Var(\frac{p(\tilde{y})}{p(\tilde{y}|x)}) + \Var(\frac{p(\tilde{y}|x)}{p(\tilde{y})})\\
& \leq & 2(\Var(p(\tilde{y})) + \Var(p(\tilde{y}|x)))
\end{array}
\end{equation}

\section{Approximation for $\I(Y;\tilde{Y})$}
\label{proof(2)}
Consider
\begin{equation}
\I(Y;\tilde{Y})=\int_{\tilde{Y}} \int_{Y}p(y,\tilde{y}) \log(\frac{p(y,\tilde{y})}{p(y)p(\tilde{y})})dy d\tilde{y}.
\end{equation}
which can be regarded as $\mathbb{E}Z$ for some random variable $Z$ with probability density:
\begin{equation}
p_Z(\log(\frac{p_{\tilde{Y},Y}(y,\tilde{y})}{p_{\tilde{Y}}(\tilde{y})p_Y(y)}))=p_{\tilde{Y},Y}(y,\tilde{y}).
\end{equation}

By Appendix~\ref{subexp}, $Z$ is subexponential, so by Bernstein's inequality(see \citet{vershynin_2018}) we have, for sample size large enough, a high probability bound guarantees the empirical approximation of $\I(Y;\tilde{Y})$:
\begin{equation}
\sum_{i=1}^n \log(\frac{p(y_i,\tilde{y_i})}{p(y_i)p(\tilde{y_i})}).
\end{equation}
which has another empirical version:
\begin{equation}
\sum_{i=1}^n \log(\frac{\sum_{j=1}^n p(y_i,\tilde{y_i}|x_j)}{p(y_i)p(\tilde{y_i})}) = \sum_{i=1}^n \log(\frac{\sum_{j=1}^n p(y_i|\tilde{y_i},x_j)p(\tilde{y_i}|x_j)}{p(y_i)p(\tilde{y_i})})= \sum_{i=1}^n \log(\frac{\sum_{j=1}^n p(y_i|x_j)p(\tilde{y_i}|x_j)}{p(y_i)p(\tilde{y_i})}).
\end{equation}
Here we make an assumption that for the models we trained over time, it's output $\tilde{Y}$ is approximately uniform distributed over the finite labels. So we treat $p(y_i)p(\tilde{y_i})$ as constant for all $i$. Also note that $p(\tilde{y_i}|x_j)$ is given by the model $f(\tilde{y_i}|x_j;w_t)$.

The prediction $\tilde{y}_i$ satisfies:
\begin{equation}
 \tilde{y}_i = \argmax_{y} f(y|x_i,w_t),
\end{equation}

So $\I(Y;\tilde{Y})$ is now of the form:
\begin{equation}
\begin{array} {lcl}
\I(Y;\tilde{Y}) & \geq & A + \sum_{i=1}^n \log(\sum_{j=1}^n p(y_i|x_j)f(\tilde{y_i}|x_j;w_t))\\
& \geq & A + \sum_{i=1}^n \log(p(y_i|x_i)f(\tilde{y_i}|x_i;w_t))\\
& \geq & A + \log(\prod_{i=1}^n f(\tilde{y_i}|x_i;w_t)) + \sum_{i=1}^n \log(p(y_i|x_i))\\
& \geq & A + \log(\prod_{i=1}^n f(y_i|x_i;w_t)) + \sum_{i=1}^n \log(p(y_i|x_i))
\end{array}
\end{equation}
with high probability for some constant $A$.

\section{Continuous entropy}
\label{ctsinfo}

The natural definition of a mutual information of a discrete random variable $X$  to $g(x)$ where $g$ is a deterministic function, if we try to define it at all, is as follows:
\begin{equation}
\I(X; g(X)) = -\sum_x p(x)\log p(x),
\end{equation}
where for simplicity we assume the range of $X$ is in $\mathbb{R}$.

If instead we consider $X$ has a continuous range, for example $\mathbb{R}$, we can take a uniform mesh over $\mathbb{R}$ with interval $\Delta$ and do the approximation:
\begin{equation}
p(X \in [x,x+\Delta]) = \int_{x}^{x+\Delta} f(a) da \approx f(x)\Delta.
\end{equation}

And therefore:
\begin{equation}
\I(X; g(X)) \approx -\sum_{i=-\infty}^{i=\infty} f(x_i)\Delta \log(f(x_i)\Delta),
\end{equation}
where $\{[x_i,x_i+\Delta]\}$ is the mesh on $\mathbb{R}$. We expect this estimation to be precise if we take $\Delta \rightarrow \infty$.

But this limit is different from the intuitive definition of differential entropy $\I(X)$, provide it exists:
\begin{equation}
\I(X)=\int_{\mathbb{R}} f(x)\log(f(x))=-\lim_{\Delta \rightarrow 0} \sum_{i=-\infty}^{i=\infty} f(x_i)\Delta \log(f(x_i)).
\end{equation}

Intuitively the $\log$ term in (3) will blow down as $\Delta \rightarrow 0$.

In practice people estimate the mutual information by (3) but it doesn't yield a meaningful quantity. In particular, we don't know whether (3) converges or not if we take $\Delta \rightarrow 0$.

Here key point here is that for general random variables $X,Y$, if $p(x,y)$ is not degenerate over some open interval $U \subset \mathbb{R}$, then by inverse function theorem, there exists some invertible relationship between $X,Y$ in $U$ and we need infinite amount of information to describe what is happening in $U$. 

But if we consider instead the mutual information between a discrete random variable $Y$ with finite range and a continuous random variable $X$ with continuous density, then the estimation would take the form:
\begin{equation}
\I(X;Y)  \approx  \sum_{i=1}^n \sum_{j=-\infty}^{\infty} p(x_j,y_i)\Delta \log(\frac{p(x_j|y_i)\Delta}{\sum_{l}p(x_j|y_l)\Delta}),
\end{equation}
which has a limit:
\begin{equation}
\I(X;Y) =  \sum_{i=1}^n \int_{\mathbb{R}} f(x,y_i) \log(\frac{f(x|y_i)}{\sum_{l}f(x|y_l)})dx.
\end{equation}
So if the analytical form (6) of $\I(X;Y)$ is finite, we know our practical approximation is meaningful.

In practice, the true quantity $\I(X;Y)$ is usually finite. For example, in MNIST, both image $X$ and its label $Y$ are essentially discrete so their mutual information can be defined in a strict discrete sense.

As a conclusion, it's always sensible to define the mutual information between discrete finite random variable and continuous random variable in an exact integral form.

\section{Heat Equation}
\label{heat}

Given a Cauchy problem for heat equation:

\begin{equation}
\begin{array} {lcl}
u_t & = ku_{xx} \\
u(x,0) & = \phi(x)
\end{array}
\end{equation}

The solution is known as:
\begin{equation}
u(x,t)=\frac{1}{\sqrt{4\pi kt}}\int^{\infty}_{-\infty} e^{-\frac{(x-y)^2}{4kt}} \phi(y)dy = K(x,t)*\phi(x),
\end{equation}

where $K(x,t) = \frac{1}{\sqrt{4\pi kt}}e^{-\frac{x^2}{4kt}}.$

Without losing generality assuming $k=1$, consider:
\begin{equation}
||K(x,t)||_2^2 = Ct^{-1}\int_{-\infty}^{\infty} e^{-\frac{x^2}{2t}}dx.
\end{equation}
Substitude $x' = t^{-1/2}x$ gives:
\begin{equation}
||K(x,t)||_2^2 = Ct^{-1/2}\int_{-\infty}^{\infty} e^{-\frac{x'^2}{2}}dx'=C't^{-1/2}.
\end{equation}

\section{Convergence in Banach Space}
\label{operator}

\begin{claim}
Consider Holder Space $C^{0,1}(\overline{U})$, where $\overline{U}$ is the closure of some bounded open set $U$, with equiped Holder norm $|\mathcal{L}| = \alpha \sup_{x \in U}|\mathcal{L}(x)| + \sup_{x,y \in U, x \neq y} \{\frac{|\mathcal{L}(x)-\mathcal{L}(y)|}{|x-y|}\}$, here $\alpha$ is some positive scalar. If $\mathcal{L} \in C^{0,1}(\overline{U})$ and $|\mathcal{L}|<1$, then there exists $\mathcal{B}$ such that $\mathcal{B}(I+\mathcal{L}) = (I+\mathcal{L})\mathcal{B} = I$.
\end{claim}
\begin{proof}

It's well known that $C^{0,1}(\overline{U})$ is a Banach space (\citet{lax2002functional}). Note that $\alpha$ can be scale down to an arbitrary small constant in practice, which has no influence to our proof.

Define
\begin{equation} \label{eq:banachseq}
\mathcal{B} = \sum_{n=0}^{\infty} (-\mathcal{L})^n.
\end{equation}
Since $|\mathcal{L}|<1$, the sequence in (\ref{eq:banachseq}) is a Cauchy sequence. So it coverges in Banach space. Convergence sequence can be multiplied termwise, it follows that
\begin{equation} \label{eq:banachseq}
\mathcal{B}\mathcal{L} = \mathcal{L}\sum_{n=0}^{\infty} (-\mathcal{L})^n=-\sum_{n=1}^{\infty} (-\mathcal{L})^n=-(\mathcal{B}-I).
\end{equation}
So $\mathcal{B}(I+\mathcal{L})= I$. The other equality can be shown similarly.
\end{proof}

\section{Subexponential}
\label{subexp}

\begin{claim}
Let $A,B$ be two discrete real-valued random variable with probability density $p_A, p_B$. Then the real valued random variable $Z$ with probability density $p_Z(\log(\frac{p_{AB}(a,b)}{p_A(a)p_B(b)}))=p_{AB}(a,b)$ is subexponential.
\end{claim}

\begin{proof}

Let $s = \log(\frac{p_{AB}(a,b)}{p_A(a)p_B(b)})$, then
\begin{equation}
e^s = \frac{p_{AB}(a,b)}{p_A(a)p_B(b)} \leq \frac{p_{AB}(a,b)}{p_{AB}(a,b)^2} = \frac{1}{p_{AB}(a,b)},
\end{equation}
and
\begin{equation}
e^s = \frac{p_{AB}(a,b)}{p_A(a)p_B(b)} \geq p_{AB}(a,b).
\end{equation}
It follows that:
\begin{equation}
p_{AB}(a,b) \leq \min \{e^{-s}, e^s\},
\end{equation}
which implies:
\begin{equation}
p_Z(s) = p_{AB}(a,b) \leq \min \{e^{-s}, e^s\},
\end{equation}

Then the tail bound satisfies:
\begin{equation}
\begin{array} {lcl}
\mathbb{P}(|Z| \geq t) & = & \int_t^{\infty} p_Z(s)ds + \int_{-\infty}^{-t} p_Z(s)ds \\
& \leq &  \int_t^{\infty} e^{-s}ds + \int_{-\infty}^{-t} e^sds \\
& = & 2e^{-t}.
\end{array}
\end{equation}
\end{proof}

\end{document}